\documentclass{bmvc2k}

\title{Morphological Network: How Far Can We Go with Morphological Neurons?}

\addauthor{Ranjan Mondal}{https://ranjanz.github.io}{1}
\addauthor{Sanchayan Santra}{https://san-santra.github.io}{2}
\addauthor{Soumendu Sundar Mukherjee}{https://soumendu041.gitlab.io}{3}
\addauthor{Bhabatosh Chanda}{bchanda57@gmail.com}{3}

\addinstitution{
Samsung Research \\
Bangalore \\
India 
}
\addinstitution{
Institute for Datability Science\\
  Osaka University\\
 Osaka, Japan
}
\addinstitution{
Indian Statistical Institute\\
Kolkata\\
India
}

\runninghead{Mondal, Santra, Mukherjee and Chanda}{Morphological Network}

\def\eg{\emph{e.g}\bmvaOneDot}

\def\etal{\emph{et al}\bmvaOneDot}

\usepackage{amsthm}
\usepackage{booktabs}
\usepackage{multirow}
\usepackage{array}
\usepackage{enumerate}

\newtheorem{lemma}{Lemma}
\newtheorem{theorem}{Theorem}
\newtheorem{proposition}{Proposition}
\newtheorem{definition}{Definition}

\usepackage{amsmath,amsfonts,bm}

\newcommand{\newterm}[1]{{\bf #1}}

\def\figref#1{figure~\ref{#1}}
\def\Figref#1{Figure~\ref{#1}}

\def\eqref#1{equation~\ref{#1}}

\def\1{\bm{1}}

\def\vs{{\bm{s}}}

\def\vw{{\bm{w}}}
\def\vx{{\bm{x}}}

\def\evs{{s}}

\def\evx{{x}}

\DeclareMathAlphabet{\mathsfit}{\encodingdefault}{\sfdefault}{m}{sl}
\SetMathAlphabet{\mathsfit}{bold}{\encodingdefault}{\sfdefault}{bx}{n}

\newcommand{\R}{\mathbb{R}}

\begin{document}

\maketitle

\begin{abstract}
Morphological neurons, that is morphological operators such as dilation and erosion with learnable structuring elements, have intrigued researchers for quite some time because of the power these operators bring to the table despite their simplicity. These operators are known to be powerful nonlinear tools, but for a given problem coming up with a sequence of operations and their structuring element is a non-trivial task. So, the existing works have mainly focused on this part of the problem without delving deep into their applicability as generic operators. A few works have tried to utilize morphological neurons as a part of classification (and regression) networks when the input is a feature vector. However, these methods mainly focus on a specific problem, without going into generic theoretical analysis. In this work, we have theoretically analyzed morphological neurons and have shown that these are far more powerful than previously anticipated. Our proposed morphological block, containing dilation and erosion followed by their linear combination, represents a sum of hinge functions. Existing works show that hinge functions perform quite well in classification and regression problems. Two morphological blocks can even approximate any continuous function. However, to facilitate the theoretical analysis that we have done in this paper, we have restricted ourselves to the 1D version of the operators, where the structuring element operates on the whole input. Experimental evaluations also indicate the effectiveness of networks built with morphological neurons, over similarly structured neural networks.

\end{abstract}

\section{Introduction}
Mathematical morphology is a set and lattice theoretic technique for the analysis of geometrical structures. Although it originated from the theoretical study of the geometry of porous materials, it is currently extensively used in the domain of digital image processing. 
It serves as a non-linear tool for processing digital images. 
Morphological operators decompose objects or shapes into meaningful parts which helps in understanding them in terms of the elements. Since the identification of objects and their features are directly correlated with their shapes and arrangement, morphological methods are quite suited for visual tasks~\cite{Haralick1992computer}. 
However, coming up with a sequence of transformation and their parameters for a given problem is not straightforward and requires expert knowledge of the problem. 
To this end, researchers have tried to automatically learn the parameters (structuring element) of mathematical morphology operators (i.e., dilation and erosion). These learnable structures are termed as \emph{morphological neurons}. Although this idea is motivated by the use of these operators in image processing, they are quite generic and can be used for tasks like classification~\cite{ritter_introduction_1996,sussner_morphological_2011} and regression~\cite{de_a._araujo_morphological_2012}. on applying morphological neurons for specific tasks.  Theoretical analysis of this structure and its properties are lacking in the literature. 
Intending to fill this gap, in this paper, we have theoretically analysed the properties of morphological neurons and shown that a specific arrangement of the neurons can approximate any continuous function. To be more precise, we have defined a structure called a \emph{Morphological Block} and shown that a sequence of two morphological blocks can work as a universal approximator. However, to facilitate the theoretical analysis, in this paper, we have restricted ourselves to the 1D version of the morphological operators, where the operators work over the whole input at once, not locally. 



The contributions of this work can be summarized as follows.
\begin{enumerate}
\item We have theoretically analyzed the properties of the morphological neurons and have shown that not all sequences of neurons are useful as some of them may be represented using fewer neurons.
\item One such useful sequence is a layer of both dilation and erosion operations, followed by a layer computing a linear combination of the outputs. We call this a \emph{Morphological Block}.
\item We have shown that a morphological block represents a sum of hinge functions. Sum hinge functions work well for tasks like regression, classification and function approximation~\cite{breiman1993hinging}.
\item We have proved that a sequence of two morphological blocks can approximate any continuous function over arbitrary compact sets.
\item We have shown due to the nonlinear nature of these operators, the neurons can learn more  complex decision boundaries with a similar number of parameters. 
\end{enumerate}

The rest of the paper is organized as follows. Existing works that has experimented with morphological neurons are briefly outlined in Section~\ref{sec:rel_work}. In Section~\ref{sec:morph_net}, we describe morphological neurons and theoretically analyse their properties. Section~\ref{sec:results} provides empirical validation of the proposed structure. 
Finally, concluding remarks are presented in Section~\ref{sec:conclustion}.

\section{Related works}
\label{sec:rel_work}
The use of morphological operations in a learning framework is first proposed by Davidson and Hummer~\cite{davidson_morphology_1993} in their effort to learn the structuring element of dilation operation on images. A similar effort has been made to learn the structuring elements in more recent work by Masci \etal{}~\cite{masci2013learning}. The use of morphological neurons for problems other than images is first proposed by Ritter and Sussner~\cite{ritter_introduction_1996}. They propose to use single-layer network architecture and focused only on the binary classification task. To classify the data, their proposed network is able to learn two axis-parallel hyperplanes as the decision boundary. This single-layer architecture has later been extended to two-layer architecture by Sussner~\cite{sussner_morphological_1998}. This two-layer architecture can learn multiple axis-parallel hyperplanes, and therefore is able to solve arbitrary binary classification tasks. But, in general, the decision boundaries may not be axis-parallel, and so, a large number of hyperplanes may need to be learned by the network. So, Barmpoutis and Ritter~\cite{barmpoutis_orthonormal_2006} proposed to learn an additional rotational matrix that rotates the input before trying to classify data using axis-parallel hyperplanes. 
In a separate work by Ritter \etal{}~\cite{ritter_two_2014} the use of $L^1$ and $L^\infty$ norm has been proposed as a replacement of the $\emph{min/max}$ operation of dilation and erosion in order to smooth the decision boundaries. 
Ritter and Urcid~\cite{ritter_lattice_2003} introduced the dendritic structure of biological neurons to the morphological networks. This structure creates hyperbox-based decision boundaries instead of hyperplanes. The authors have proved that hyperboxes can estimate any compact region and, thus, any two-class classification problem can be solved. A generalization of this structure to the multiclass case has also been done by Ritter \etal{}~\cite{ritter_learning_2007}. Experimentation with network architecture has also been attempted by Sussner and Esmi~\cite{sussner_morphological_2011}, where they propose a new structure called morphological neurons with competitive learning. In this setting, the \emph{argmax} operator is utilized at the output of several neurons to implement the winner-take-all strategy. The authors claim with this setup the network is able to learn complex decision boundaries. 
=
Methods mentioned till now, employ special optimization techniques to learn the parameters since the $\max$ and $\min$ operations employed by dilation and erosion are not differentiable. So, altogether different strategies have been proposed to overcome this issue. 
Ara\'{u}jo~\cite{de_a._araujo_morphological_2012} utilized network architecture similar to morphological neurons with competitive learning to forecast stock markets. The \emph{argmax} operator was replaced with a linear activation function so that the network is able to regress forecasts and the gradient descent could be utilized for training. 
For morphological neurons with dendritic structure, Zamora and Sossa~\cite{zamora_dendrite_2017} proposed to replace the \emph{argmax} operator with a softmax function, in order to utilize the gradient descent optimizer. However, more recent methods don't consider this a hindrance. The gradient is computed where it is possible and taken to be 0 at other places.

The more recent works employing morphological neurons take altogether different approaches in using morphological operations.
Franchi \etal{}~\cite{franchi2020deep} proposed to utilize morphological operations as layers within neural networks. They have shown the pooling layer in CNNs can be replaced with a learned morphological pooling, and using CNNs with only morphological layers works well in denoising images. 
Nogueira \etal{}~\cite{nogueira2019introduction} proposed to utilize morphological operations to be able to learn novel deep features while training the network end-to-end with gradient descent. The authors' have shown experimentally that these features work well for different image classification tasks. 
Islam \etal{}~\cite{aminul2019deep} proposed using morphological hit-or-miss transform to build networks. Mondal \etal{}~\cite{mondal2020image} introduced the opening-closing network for image de-raining and dehazing using morphological opening and closing operations. Limonova \etal{}~\cite{limonova2020bipolar} proposed new morphological neurons called, bipolar morphological neurons. The authors claim to achieve better recognition results compared to neural networks. 

Although morphological neurons have been utilized in various ways for specific applications, their generic theoretical justification is scarce to date. It is still an open question how morphological networks should be designed so that they become a generic tool that can solve any learning problem.
In the following subsections, we have tried to answer these questions.

\section{Morphological Neurons}
\label{sec:morph_net}

In this section, we first introduce morphological neurons and their properties. Then we define the Morphological Block and show that it computes a sum of hinge functions. As pointed out in \cite{breiman1993hinging}, hinge functions are a powerful alternative for classification, regression and function approximation. Although the function approximation capability of a morphological block is not known, we have proved that by using two morphological blocks sequentially, we can approximate any function.

\subsection{Dilation and Erosion neurons}
\label{sec:de_neurons} 
Dilation and Erosion neurons are the two most basic morphological neurons because all other morphological operations can be represented as a composition of these two. Note that, we are utilizing the 1D version of these operations to facilitate theoretical analysis.
Given an input $\vx \in \R^d$ and a structuring element $\vs \in \R^{d}$, the operation of \newterm{dilation} ($\oplus$) and \newterm{erosion} ($\ominus$) neurons are defined, respectively, as 
\begin{align}
    \vx \oplus \vs  = \max_{k}\{\evx_k+\evs_k\}, \\
    \vx \ominus \vs  = \min_{k}\{\evx_k-\evs_k\},  \label{eq:dilationerosion}
\end{align}
where  $\evx_k$ denotes $k^{th}$ element of  input vector $\vx$. After computing dilation and erosion we may set a \textit{limiter} or \textit{bias}, say, $\evs_{d+1}$ to compute final output from dilation and erosion neurons by $\max\{ {\vx \oplus \vs }, \evs_{d+1} \}$ and $\min\{ {\vx \ominus \vs }, -\evs_{d+1} \}$ respectively. Note that this ensures $\evs_{d+1}$ to be the lower bound of the output of the dilation neuron, whereas it is the upper bound for the erosion; hence, the term `limiter'. 
Alternatively, we can write it as follows. Let 0 is appended to the input $\vx$, i.e, $\vx' = [\vx, 0]^T$ and  $\evs_{d+1}$ is appended to $\vs$, then
\begin{align}
     \max\{ {\vx \oplus \vs }, \evs_{d+1} \}&=\max\{\max_{k}\{\evx_k+\evs_k\},\evs_{d+1}\} =\max_{{k=1,\dots,d+1}}\{\evx'_k+\evs'_k\}
                    =\vx' \oplus \vs'       \label{eq:dilation_bias} 
\end{align}
Where $\evs'_k$ is a element of structuring element  $\vs'$. Similarly we can get $\vx' \ominus \vs'$. It may be argued that ${d+1}^{th}$ component is selected if the input has no effect on the output or the function. 
In these neurons, the structuring element ($\vs'$) is learned in the training phase.

The $\max$ and $\min$ operators used in the dilation and erosion neurons are only piece-wise differentiable. As a result, only a single element of the structuring element is updated at each iteration. To overcome this problem we propose to use the soft version of $\max$ and $\min$~\cite{cook2011basic} to define \emph{soft dilation} and \emph{soft erosion} neurons as follows.
\begin{align}
    \vx' \hat{\oplus} \vs'  = {\frac{1}{\beta}\log\left( \sum_{k+1} e^{(\evx'_k+\evs'_k)\beta}\right)},  \\
    \vx' \hat{\ominus} \vs'  = -{\frac{1}{\beta}\log\left(\sum_{k+1} e^{(\evs'_k-\evx_k)\beta}\right)},
    \label{eq:softdilationerosion}
\end{align}
where $\hat{\oplus}$ and $\hat{\ominus}$ denote the soft dilation and soft erosion, respectively, and $\beta$ is the ``hardness'' of the soft operations. The soft version can be made close to its ``hard'' counterpart by making $\beta$ large enough \cite{cook2011basic}. Henceforth, for notational convenience, we use $\vx$ and $\vs$ to denote input and structuring element respectively for dilation (or erosion) with a limiter. In other words, the dilation and erosion neurons include the limiter. 

\subsection{Gradient of Morphological Neurons} 
\label{sec:decision_boundary}
Network build using morphological neurons can be trained using the backpropagation algorithm, provided we are able to find their derivative.
The $\max$ and $\min$ operations of dilation and erosion, respectively, are not differentiable. To be precise, they are not differentiable when the arguments to the $\max$ or $\min$ operation are equal. However, this rarely occurs in practice. So, we may define the derivative of the dilation and erosion operation in the following way.
\begin{align}
    \frac{\partial z^+}{\partial s_i} = 
    \begin{cases}
        1 & \text{if $x_i+s_i$ is the max}, \\
        0 & \text{otherwise}.
    \end{cases} 
    \quad
    \frac{\partial z^-}{\partial s_i} = 
    \begin{cases} 
        1 & \text{if $x_i-s_i$ is the min}, \\
        0 & \text{otherwise}.
    \end{cases}
    \label{gradeq}
\end{align}
So, the computed gradient is non-zero only in the element for which the maximum (or minimum) is attained. For this reason, the computed loss or error affects only one element of the structuring element for a given sample. 
As a result, only a single neuron is activated in the network and, consequently, only a single weight ($s_i$) is updated at a time. 
This may result in slow convergence of the training of the network. In practice, training a large morphological network is very slow. Soft morphological neurons mitigate this issue to some extent.


\subsection{Equivalence of configurations}
The morphological neurons may be arranged in different ways to create a network with the aim of solving a particular task. But not all of these configurations are useful and some of them may be spurious. The following are true for different network architectures.
\begin{theorem}
If we denote $D_{m_1}E_{m_2}$ as a layer with $m_1$ dilation neurons and $m_2$ erosion neurons and $L$ as a linear combination layer, the following may be said about their configurations. 
\begin{enumerate}[(i)]
    \item The architecture $D_{m_1}E_{0}\rightarrow D_{m_2}E_{0} \rightarrow \cdots \rightarrow D_{m_\ell} E_{0}$ consisting only of dilation layers is equivalent to the architecture $D_{m_\ell}E_{0}$ with a single dilation layer. A similar statement is true if one considers architectures with only purely erosion layers.
    \item The architecture $D_1E_1 \rightarrow D_1$ is not equivalent to $D_1E_0$. Similarly, it is not equivalent to $D_0E_1$, and, consequently, the architectures $D_1E_1 \rightarrow D_1E_1$ and $D_1E_1$ are not equivalent.
    \item The architecture $D_1 E_1 \rightarrow D_1 \rightarrow L$ is not equivalent to $D_1E_0 \rightarrow L$.
    \item The architecture $D_{2}E_{0}\rightarrow D_{0}E_{2} \rightarrow D_1$ is not equivalent to $D_{2}E_{0} \rightarrow D_1$.
\end{enumerate}
\end{theorem}
The proof is provided in the supplementary material.

\subsection{Morphological block}
\label{sec:SL_morph}
Here we define Morphological Block, which is one configuration that can be utilized as a building block for making more complex networks. A Morphological block consists of a layer with dilation and erosion neurons followed by a linear combination of their outputs (\Figref{fig:single_layer_network}). We call the layer of dilation and erosion neurons the \newterm{dilation-erosion layer} and the following layer as the \newterm{linear combination layer}. 
Let us consider a morphological block with $n$ dilation neurons and $m$ erosion neurons in the dilation-erosion layer followed by $c$ neurons in the linear combination layer. Let $\vx \in \R^d$ be the input to the network and $z_{i}^+$ and $z_{j}^-$ be the output of the $i^{th}$ dilation neuron and the $j^{th}$ erosion neuron respectively: 
\begin{align}
    z_{i}^+ &= \vx \oplus \vs_i^+ , \\ 
    z_{j}^- &= \vx \ominus \vs_j^-  \label{eq:erosion_net},
\end{align}
where $\vs_i^+$ and $\vs_j^-$ are the structuring elements of the respective neurons. Note that $i \in \{1, 2, \ldots, n\}$ and $j \in \{1, 2, \ldots, m\}$. The final output of a node in the linear combination layer is computed as 
\begin{equation}
    \mathcal{M}(\vx)=\sum_{i=1}^{n} z_{i}^+\omega_i^{+}+\sum_{j=1}^{m} z_{j}^- \omega_j^{-},
    \label{eq:total}
\end{equation}
where $\omega_i^{+}$ and $\omega_j^{-}$ are the weights of the combination layer. When the network is trained, it learns all $\vs_i^+$, $\vs_j^-$, $\omega_i^{+}$ and $\omega_j^{-}$. 

\begin{figure}
    \centering
    \includegraphics[width=0.8\linewidth ]{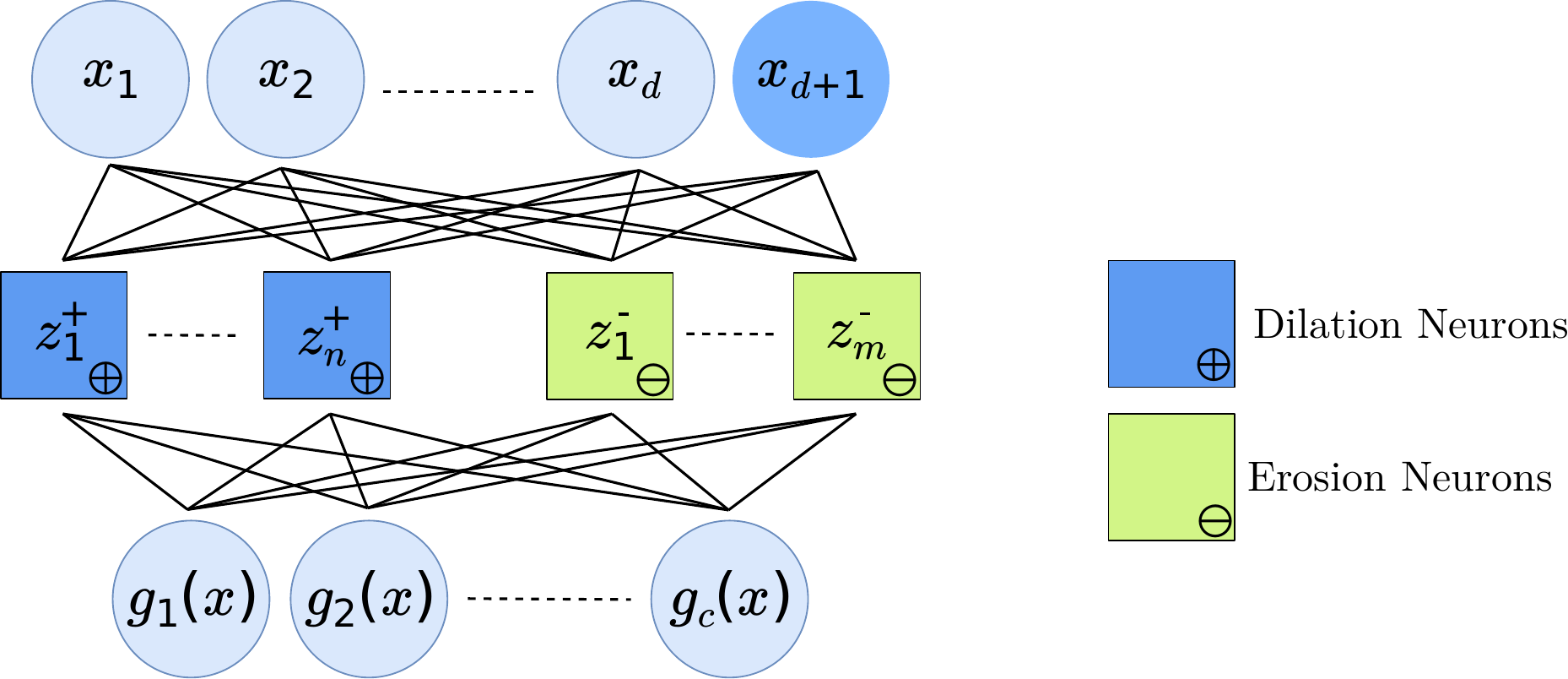}   
    \caption{Architecture of single layer morphological block. It contains an input layer, a dilation-erosion layer with $n$ dilation and $m$ erosion neuron and a linear combination layer with $c$ neurons producing the output. The limiter associated with the input $x_{d+1}=0$.}
    \label{fig:single_layer_network}
\end{figure}

\subsection{Morphological block as a sum of hinge functions}
In this subsection, we show that the simple morphological block can be represented as a sum of hinge functions. 
Hinge functions provide a powerful representation for problems like classification, and regression tasks~\cite{breiman1993hinging}. 
Additionally, we try to develop the intuition behind the morphological block with a toy example.

\begin{definition}[$k$-order Hinge Function \cite{wang2005generalization}]
    A $k$-order hinge function $h^{(k)}(\vx)$ consists of $(k+1)$ hyperplanes continuously joined together. It may be defined as 
    \begin{equation}
        h^{(k)}(\vx) = \pm \max\{\vw_1^{T}\vx+b_1, \vw_2^{T}\vx+b_2, \ldots, \vw_{k+1}^{T}\vx+b_{k+1}\} 
        \label{eq:k-order_hinge}
    \end{equation}
\end{definition}

\begin{proposition}
\label{th:gx_sum_hinge}
The function computed by a Morphological Block (denoted by $\mathcal{M}(\vx)$) with $n$ dilation and $m$ erosion neurons followed by their linear combination, is a sum of multi-order hinge functions.
\end{proposition}
\noindent In fact, we can show that 
\begin{equation}
    \mathcal{M}(\vx) = \sum_{i=1}^{l} \alpha_{i} h^{(d)}_i(\vx), 
    \label{eq:lemma_sum_of_hinge}
\end{equation}
where $l=m+n$, $\alpha_{i} \in \{1,-1\}$ and $h^{(d)}_i(\vx), 1 \le i \le l$, are $d$-order hinge functions. The proof is given in the supplementary material.

\begin{figure}
    \centering
    \includegraphics[width=0.3\linewidth]{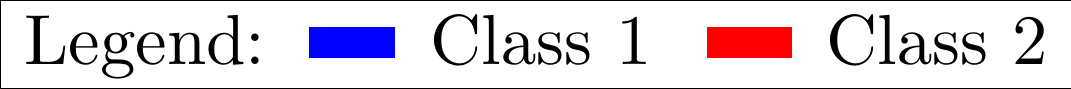}
    \smallskip
    
    \subfigure[NN-ReLU]{\includegraphics[width=0.247\linewidth]{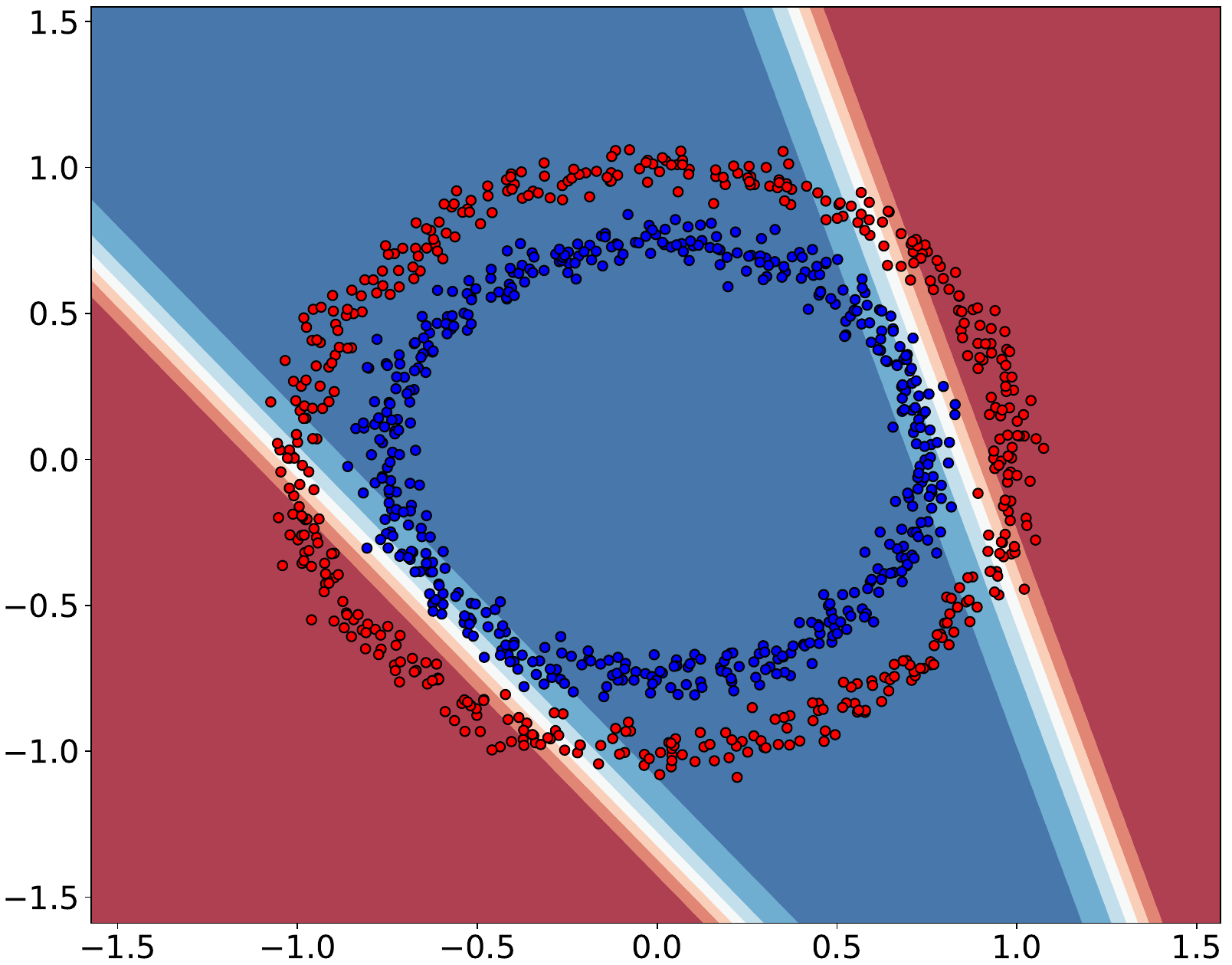}
    \label{fig:decision_nn}}
    \subfigure[Maxout Network]{\includegraphics[width=0.24\linewidth]{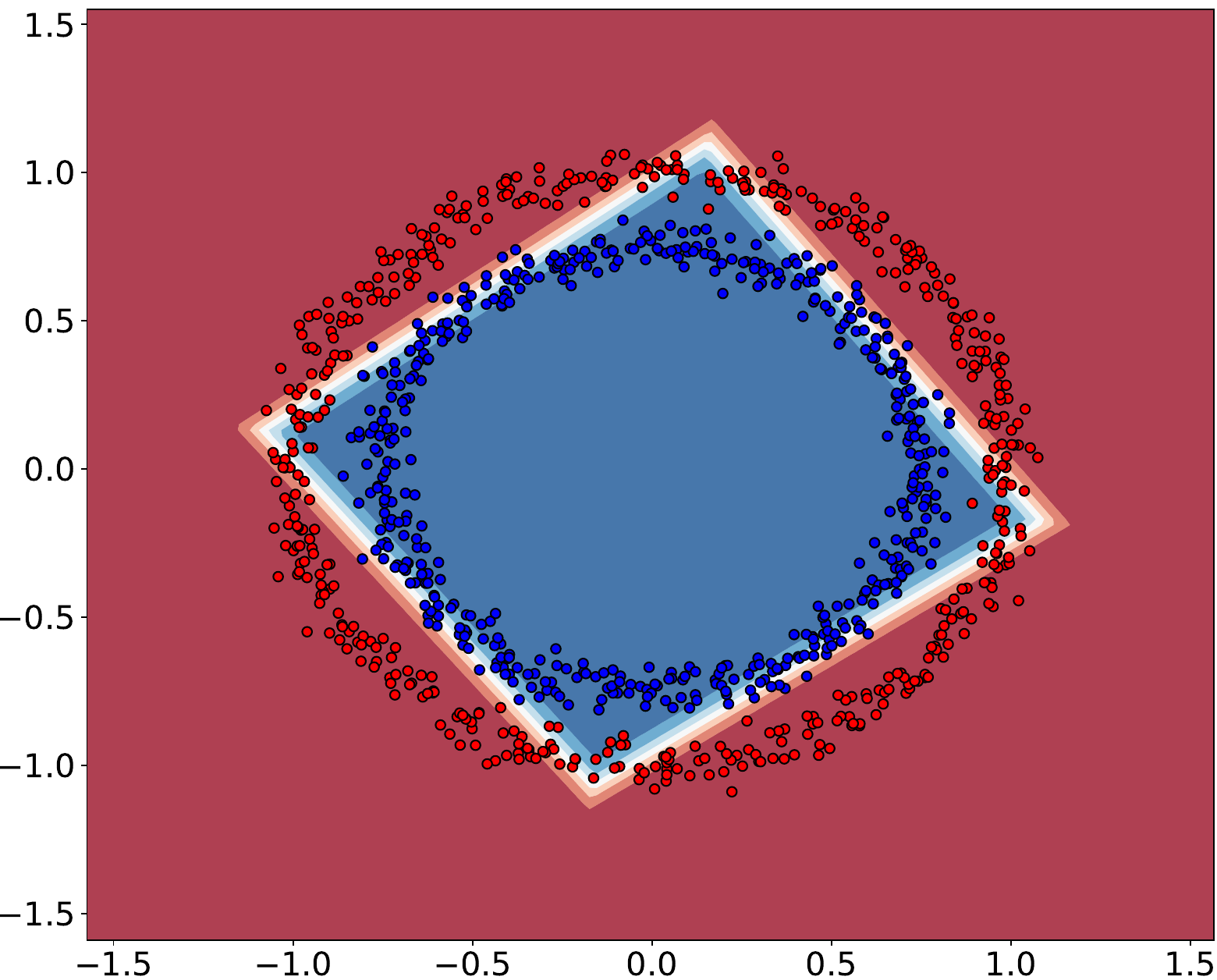}
    \label{fig:decision_maxout}}
    \subfigure[Morphological Block]{\includegraphics[width=0.216\linewidth]{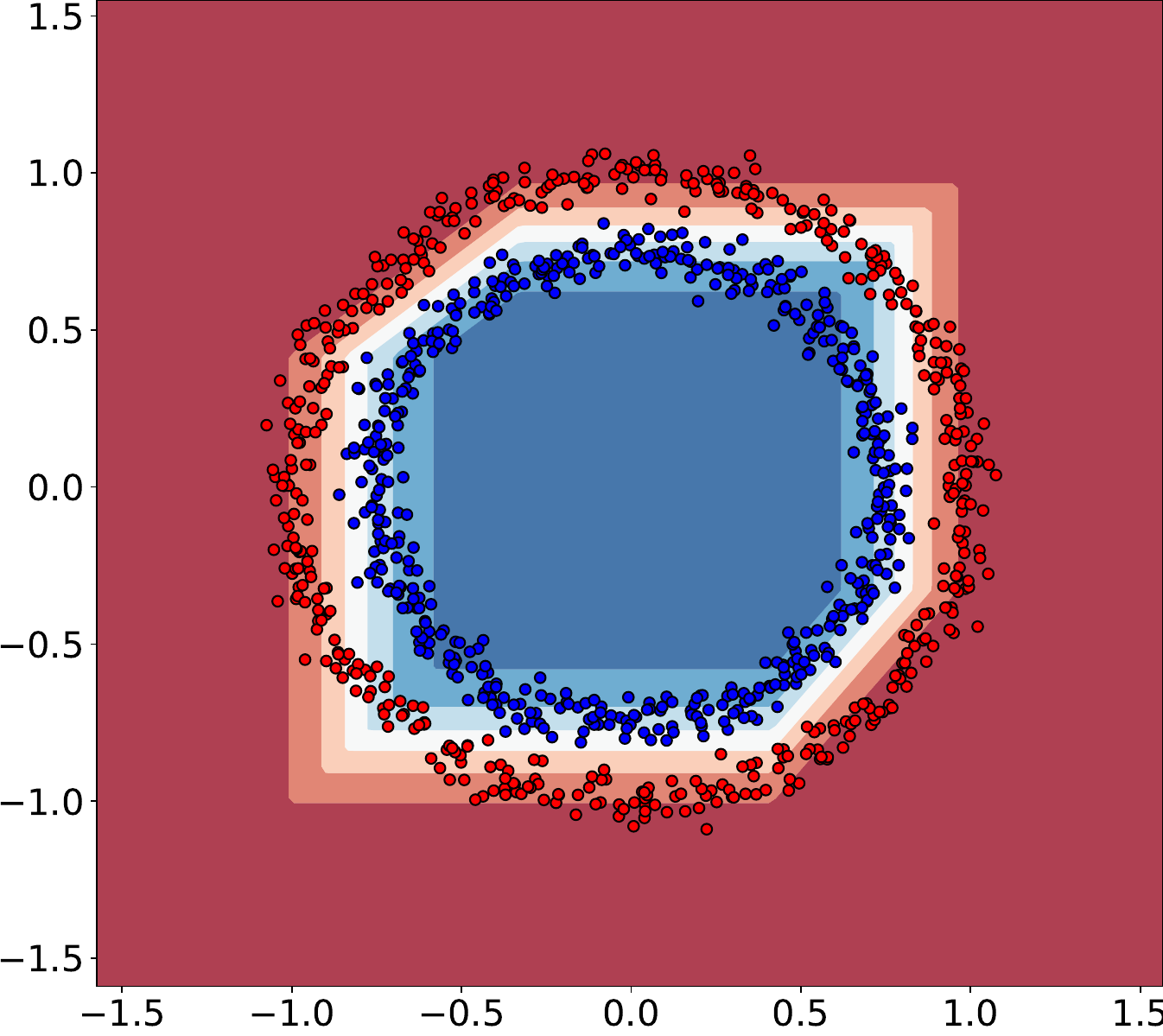}
    \label{fig:decision_morph}}
    \subfigure[Soft Morphological Block]{\includegraphics[width=0.2175\linewidth]{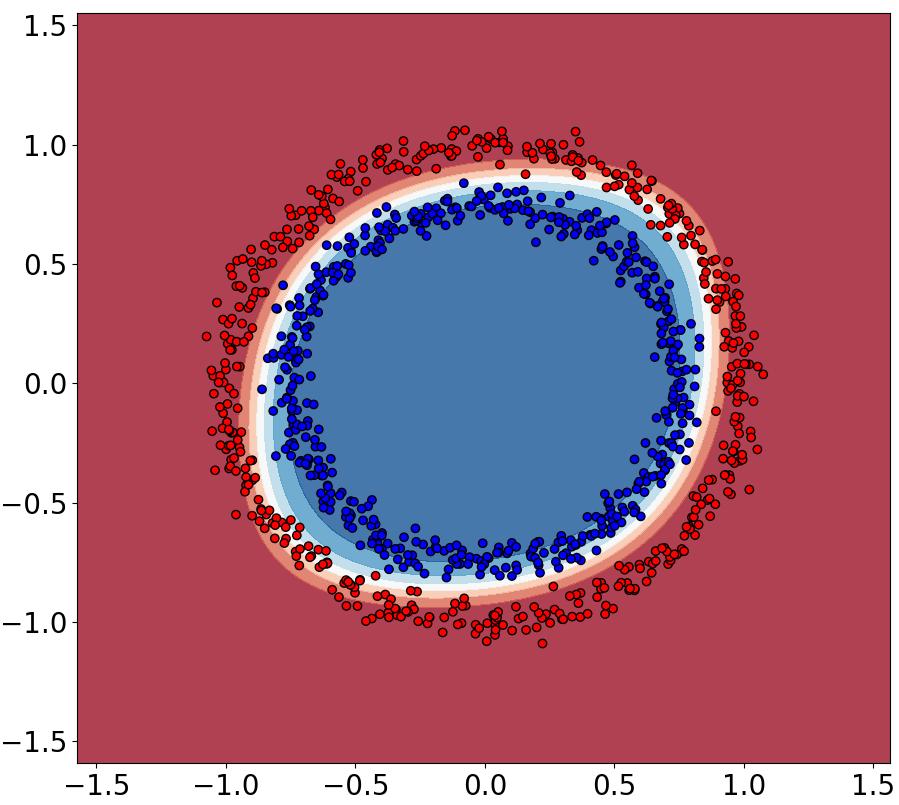}
    \label{fig:decision_morph_soft}}
\caption{Decision boundaries learned by different networks with two hidden neurons. (a)~Baseline neural network is able to learn only two planes. 
(b)~Maxout network is able to learn two more planes with the help of additional parameters. 
(c)~Morphological Block is able to learn more planes with the same number of parameters as NN-ReLU. 
(d)~Using the soft version of the block, smooths the learned decision boundary. 
This further enhances the discrimination capability of the network while retaining the same number of parameters.}
    \label{fig:pmap_circle}
\end{figure}

Since a morphological block computes a sum of hinge functions, it can potentially learn a large number of hyperplanes. It can be seen a morphological block can have maximum  $(d+1)^l-1$ hyperplanes. Out of all those, there can be almost $d! \times \binom{l}{d} \times {(d+1)^{l-d}}$ hyperplanes that are not parallel to any of the axes. (This has been further explained in the supplementary material.) These hyperplanes can act as decision boundaries as has been demonstrated experimentally using a toy dataset representing a two-class problem. 

The toy dataset contains samples that are distributed along two concentric circles, one circle for each class. The circles are centered at the origin.
We compare the results obtained by various networks with two neurons in the hidden layer. It is observed that the baseline neural network (NN-ReLU) fails to classify this data as with two hidden neurons it learns only two hyperplanes, one for each neuron (\Figref{fig:decision_nn}).
The result of maxout network\cite{goodfellow_maxout_2013} is better, because, in this case, the network learns $2k = 4$ hyperplanes as shown in \figref{fig:decision_maxout}. 
Note that with two morphological neurons in the dilation-erosion layer, our network has learned 6 hyperplanes to form the decision boundary (\Figref{fig:decision_morph}). We should get at most 8 hyperplanes from the morphological block. However, out of these only two decision boundaries are placed in any arbitrary orientation in the 2D space, while others are parallel to either of the axes. Whereas using the soft version of dilation and erosion smooths the boundary, making it aligned with the data (\Figref{fig:decision_morph_soft}). 

\subsection{Two Morphological Blocks: An universal approximator}
A single morphological block may be able approximate functions, but we don't know how well is its approximation capability (Supplementary material provides an empirical study on this). However, with two morphological blocks (applied sequentially) any function can be approximated. 
\begin{lemma}
Any linear combination of hinge functions $\sum_{i = 1}^m \alpha_i h^{(k_i)}(\vx)$ can be represented over an arbitrary compact set $K$ as a two sequential morphological block consisting of dilation neurons only.
\end{lemma} 
\begin{proof}
The proof is given in the supplementary material.
\end{proof}
\begin{theorem}[Universal approximation]
Two morphological blocks applied sequentially, can approximate continuous functions over arbitrary compact sets.
\end{theorem}
\begin{proof}
Continuous functions can be approximated over compact sets by sums of hinge functions (Theorem 3.1 of \cite{breiman1993hinging}). Therefore, by Lemma 1, it follows that any continuous function can be approximated over arbitrary compact sets by two-layer Morph-Nets.
\end{proof}


\section{Experimental Evaluation}
\label{sec:results}
To empirically evaluate the performance of our proposed Morphological block, we have done experiments using various benchmark datasets of several real-world problems. We have compared our results with similarly structured networks because the state-of-the-art methods employ more than just simple neurons to accomplish the results. Also, since our proposed morphological block utilizes 1D morphological operations, the comparison has been done with networks with 1D neurons. So, all the data has been flattened before feeding to the networks. Because of these reasons, we have evaluated our method on MNIST, Fashion-MNIST~\cite{xiao2017fashion}, CIFAR-10, and SVHN dataset only. We have refrained from using large image datasets as a 1D version of the operators won't be able to extract meaningful features from them. Yes, using the 2D version of the morphological operations would indeed be more appropriate for image data, but here our focus is the evaluation of our proposed morphological block, not its 2D version. 


\subsection{MNIST and Fashion-MNIST} 
For MNIST and Fashion-MNIST~\cite{xiao2017fashion} dataset, the network we have utilized contains an input layer and a single morphological block. Only a sigmoid activation has been utilised in the last layer, no other activation function has been used. The morphological block contains 200 dilation and 200 erosion neurons.
Table~\ref{tab:mnists} shows the accuracy of test data after training the network for 300 epochs. We get an average accuracy of $98.43\%$ and $89.84\%$, respectively, on MNIST and Fashion-MNIST datasets. Note that the reported state-of-the-art techniques make use of different data augmentation and pre-processing techniques to train the data.  We have not utilized any of such techniques, but still, we are able to get comparable results.

\begin{table}
    \centering
    \label{tab:mnists}
    \begin{tabular}[c]{lccc}
        \toprule
        \multirow{2}{*}{\textbf{Dataset}} & \multicolumn{3}{c}{\textbf{Test Accuracy}} \\ \cmidrule{2-4}
        & Morph-Net & Soft Morph-Net ($\beta=8$) & Similar Network \\ 
        \cmidrule(r){1-1}\cmidrule(lr){2-2}\cmidrule(lr){3-3}\cmidrule(l){4-4}
        MNIST & 98.39 & 98.90 & \textbf{99.79} \cite{wan2013regularization} \\ 
        Fashion-MNIST & 89.87 & \textbf{89.84} & 89.70 \cite{xiao2017fashion} \\
        \bottomrule
    \end{tabular}
    \caption{Accuracy on MNIST and Fashion-MNIST Datasets using a single hidden layer with 400 morphological neurons.}
\end{table}


\subsection{CIFAR-10 and SVHN}
CIFAR-10~\cite{krizhevsky_learning_2009} and SVHN~\cite{netzer2011reading} are two popular classification datasets that are far more challenging than the previous two. All the networks we have utilized here to report the results, follow a 3-layer architecture: input layer, hidden layer and output layer. For the Maxout network~\cite{goodfellow_maxout_2013}, we have taken $k=2$ which means each hidden neuron has two extra nodes over which the maximum is computed. For the network with (soft) morphological neurons, sigmoid activation has been utilized only in the last layer. Table~\ref{cifar10_svhn_table} shows the mean and standard deviation of test accuracy obtained over 5 runs of 300 epochs each and by varying the number of neurons in the hidden layer. It is seen from the table that the Morph-Nets achieve better accuracy for the CIFAR-10 dataset in all cases. However, for the SVHN dataset, its results are comparable with that of other networks. But for both datasets, the accuracy obtained by Morph-Net stays almost the same across training runs. That is not true for other networks.

\begin{table}
    \centering
    \label{cifar10_svhn_table}
    \scalebox{0.82}{
    \begin{tabular}{m{0.2\linewidth}cccccc}
    \toprule
    \multirow{2}{*}{\textbf{Architecture}} & \multicolumn{2}{c}{l=200} & \multicolumn{2}{c}{l=400} & \multicolumn{2}{c}{l=600}\\
    \cmidrule(r){2-3}\cmidrule(lr){4-5}\cmidrule(l){6-7}
    & CIFAR10 &SVHN & CIFAR10 &SVHN & CIFAR10 &SVHN \\
    \cmidrule(r){1-1} \cmidrule(r){2-2}\cmidrule(lr){3-3}\cmidrule(lr){4-4}\cmidrule(lr){5-5}\cmidrule(lr){6-6}\cmidrule(l){7-7}
    NN-tanh & 46.6 $\pm$ 0.06 &73.9 $\pm$ 0.12  &46.9 $\pm$ 0.04 & 73.9 $\pm$ 0.23 &48.0 $\pm$ 0.05 & 75.6 $\pm$ 0.14  \\
    NN-ReLU & 47.2 $\pm$ 0.11&64.2 $\pm$ 0.88  &48.0 $\pm$ 0.05 & 76.2 $\pm$ 0.32 &48.1 $\pm$ 0.02 & \textbf{79.5 $\pm$ 0.11}  \\
    Maxout-Network ($k = 2$)~\cite{goodfellow_maxout_2013} &46.9 $\pm$ 0.05 &69.4 $\pm$ 0.10  &48.0 $\pm$ 0.10 & 74.1 $\pm$ 0.22 & 46.4 $\pm$ 0.33 & 37.8 $\pm$ 3.15 \\
    Our& 52.0 $\pm$ 0.02& 73.4 $\pm$ 0.03 &53.6 $\pm$ 0.01 & 76.9 $\pm$ 0.03  &54.0 $\pm$ 0.02 & 78.2 $\pm$ 0.03\\
    Our (Soft: $\beta=12$, $20$)&\textbf{53.5 $\pm$ 0.04} & \textbf{74.1 $\pm$ 0.06} &\textbf{55.8 $\pm$ 0.05} &  \textbf{77.0 $\pm$ 0.05} &\textbf{56.9 $\pm$ 0.04} &78.5 $\pm$ 0.05\\
    \bottomrule
    \end{tabular}
    }
    \caption{Test accuracy achieved on CIFAR-10 and SVHN dataset by different networks when the number of neurons ($l$) in the hidden layer is varied. The value  of $\beta$ is 12 and 20 for CIFAR10 and SVHN respectively.}
\end{table}

\section{Conclusion}
\label{sec:conclustion}
In this paper, we have theoretically analysed the morphological neurons and have shown that our proposed morphological block is a good way to arrange morphological neurons. We have also shown, that a morphological block represents a sum of hinge functions and two morphological blocks can approximate any continuous function. This provides the theoretical basis that networks built with morphological neurons are equally capable and it is applicable to a variety of problems. The empirical results also show the applicability of morphological neurons. But there is huge scope for further explorations. Firstly, the networks build with morphological neurons (or morphological blocks) are very slow to train since only a small number of parameters are updated at each iteration (Remember that the gradient is 1 only where the max/min occurs). So, it may also require more number iterations to converge.
Improvements in this regard will greatly boost the scope for further exploration.
Secondly, to facilitate theoretical analysis, we have restricted ourselves to the 1D version of the morphological operations. This can be easily extended to 2D morphological operations to make CNN-like networks, but to compete with the state-of-the-art networks other advanced layers (\eg{} batch norm, drop-out) may need to be adapted for morphological neurons. 



\bibliography{egbib}
\end{document}


\maketitle

\section{Gradient of Soft maximum}
The derivative of soft dilation and erosion operation may be defined as follows.
\begin{align}
    \frac{\delta (\vx \hat{\oplus} \vs)}{\delta s_k}    &= \frac{e^{(\evx_k+\evs_k)\beta}}{\sum_i e^{(\evx_i+\evs_i)\beta}}
    \label{eq:grad_softdilation} \\
    \frac{\delta (\vx \hat{\ominus} \vs)}{\delta s_k}    &= \frac{e^{(\evs_k-\evx_k)\beta}}{\sum_i e^{(\evs_i-\evx_i)\beta}}
    \label{eq:grad_softerosion}
\end{align}

\section{Equivalence of Configurations}
In this section, we prove that some of the arrangements of morphological neurons are equivalent and can be approximated by using a fewer number of neurons. To be able to do that, we first prove a simple lemma. 
\begin{lemma}
Suppose $f$ and $g$ are real-valued functions on $\R^d$. Then $f = g$ if and only if, for all $r \in \R$, one has equality of the sub-level sets:
\[
    f^{-1}(-\infty, r] = g^{-1}(-\infty, r].
\]
\label{lemma_feqg}
\end{lemma}
\begin{proof}
The ``only if'' part is trivial. As for the ``if'' part, note that we have
\begin{align*}
    f^{-1}\{r\} = \bigcap_{n \ge 1}f^{-1}(r - 1/n, r] = \bigcap_{n \ge 1}(f^{-1}(-\infty, r] \setminus f^{-1}(-\infty, r - 1/n]).
\end{align*}
The same goes for $g$, and so, by our hypothesis,
\[
    f^{-1}\{r\} = g^{-1}\{r\} \quad \text{for all } r \in \R.
\]
Therefore, for any $x \in \R^d$, we have $x \in g^{-1}\{g(x)\} = f^{-1}\{g(x)\}$, or, in other words, $f(x) = g(x)$.
\end{proof}


\begin{theorem}
If we denote $D_{m_1}E_{m_2}$ as a layer with $m_1$ dilation neurons and $m_2$ erosion neurons and $L$ as a linear combination layer, the following may be said about their different configurations. 
\begin{enumerate}[(i)]
    \item The architecture $D_{m_1}E_{0}\rightarrow D_{m_2}E_{0} \rightarrow \cdots \rightarrow D_{m_\ell} E_{0}$ consisting only of dilation layers is equivalent to the architecture $D_{m_\ell}E_{0}$ with a single dilation layer. A similar statement is true if one considers architectures with only purely erosion layers.
    \item The architecture $D_1E_1 \rightarrow D_1$ is not equivalent to $D_1E_0$. Similarly, it is not equivalent to $D_0E_1$, and, consequently, the architectures $D_1E_1 \rightarrow D_1E_1$ and $D_1E_1$ are inequivalent.
    \item The architecture $D_1 E_1 \rightarrow D_1 \rightarrow L$ is not equivalent to $D_1E_0 \rightarrow L$.
    \item The architecture $D_{2}E_{0}\rightarrow D_{0}E_{2} \rightarrow D_1$ is not equivalent to $D_{2}E_{0} \rightarrow D_1$.
\end{enumerate}
\end{theorem}

\begin{proof}
(i) Let $x \in R^d $ be the input to the network. Let there be two networks $N_{1}$ and $N_2$. Let there be $m_1$ and $m_2$ dilated neurons in, respectively, the first and the second layers of Network $N_1$. Let the parameters of the network $N_1$ in the first layer and 2nd layer are $w^1 \in R^{d\times m_1}$  and $w^2 \in R^{l_{1}\times l_{2}}$  respectively. Whereas let  there is only  a single layer with  $m_1$ number of  dilated neurons in network $N_2$ and the parameters are denoted as $u \in R^{d\times m_2}$.  Let  $f(x) \in R^{m_2} $  and $g(x) \in R^{m_2}$ are the output from the last layer of network $N_1$ and $N_2$ respectively.

For Network $N_1$
\begin{align}
  \label{eqs0}
  &y_{j} =\max_i(x_i+w_{i,j}^1) \quad \forall j \in\{1,2,..m_1\} \\
  \label{eqs1} 
  &f_{k}(x) =\max_j(y_{j}+w_{j,k}^2)  \quad \forall j,k
\end{align}

For network $N_2$
\begin{align}
  g_{k}(x) =\max_j(x_{j}+u_{j,k}^2) \qquad \forall k,j   
\end{align}

Let
\begin{align}
S_{f}^k=\{x \mid f_{k}(x) \leq e_k; e_k \in R \} \\
S_{g}^k=\{x \mid f_{k}(x) \leq e_k; e_k \in R \}
\end{align}

For Network $N_1$
\begin{align}
\label{eqs2}
f_{k}(x) \leq e_{k} ;  \qquad  \forall k   \\
\label{eqs3}
y_i + w_{i,j}^2 \leq e_{k}   \qquad \forall k,j\\
\label{eqs4}
y_i \leq e_{k} - w_{i,j}^2  \qquad \forall k,j
\end{align}

From equation~\ref{eqs0} and equation~\ref{eqs4} we get 
\begin{align}
\label{eqs5}
\max_i(x_i+w_{i,j}^1)   \leq e_{k} - w_{i,j}^2  \qquad  \forall k,j \\
\label{eqs6}
x_i+w_{i,j}^1   \leq e_{k} - w_{i,j}^2 \qquad  \forall k,j,i 
\end{align}

\begin{align}
\label{eqs7}
x_i  \leq e_{k} - w_{i,j}^2 -w_{i,j}^1  \qquad \forall k,j,i 
\end{align}

Which means 
\begin{align}
\label{eqs8}
x_i  \leq \min_{j}(e_{k} - w_{i,j}^2 -w_{i,j}^1) \qquad \forall k,i \\
x_i  \leq e_{k}  - \max_{j}(w_{i,j}^2 +w_{i,j}^1) \qquad \forall k,i 
\end{align}

For network $N_2$ 
\begin{align}
g_{k}(x) =\max_j(x_{j}+u_{j,k}^2)           \\
x_i  \leq (e_{k} -u_{i,k}) \forall k,i 
\end{align}

To hold  the  set $S_{g}^k$  is equal to $S_{f}^k$ to   $\forall k$ 

\begin{align}
u_{i,k}=\max_{j}(w_{i,j}^2 +w_{i,j}^1) \forall i,k
\label{eqlongpath}
\end{align}
Hence, from Lemma \ref{lemma_feqg}, given a parameter  $w^1$ and $w^2$ of and 2 layer network $N_1$, there exist a equivalent single-layer network $N_2$ with dilated neurons $u$ which can represent the same  function.From the equation \ref{eqlongpath} we can see the parameters of the single-layer network can be constructed considering the longest path from input to output. Recursively we can say it holds for multiple layers. A similar argument can be given in the case of erosion layers.
Hence $D_{m_1}E_{0}\rightarrow D_{m_2}E_{0} \rightarrow \cdots \rightarrow D_{m_\ell} E_{0}$ proved

\vskip10pt
\noindent
(ii)
\begin{figure}
    \centering
    \includegraphics[width=0.7\textwidth]{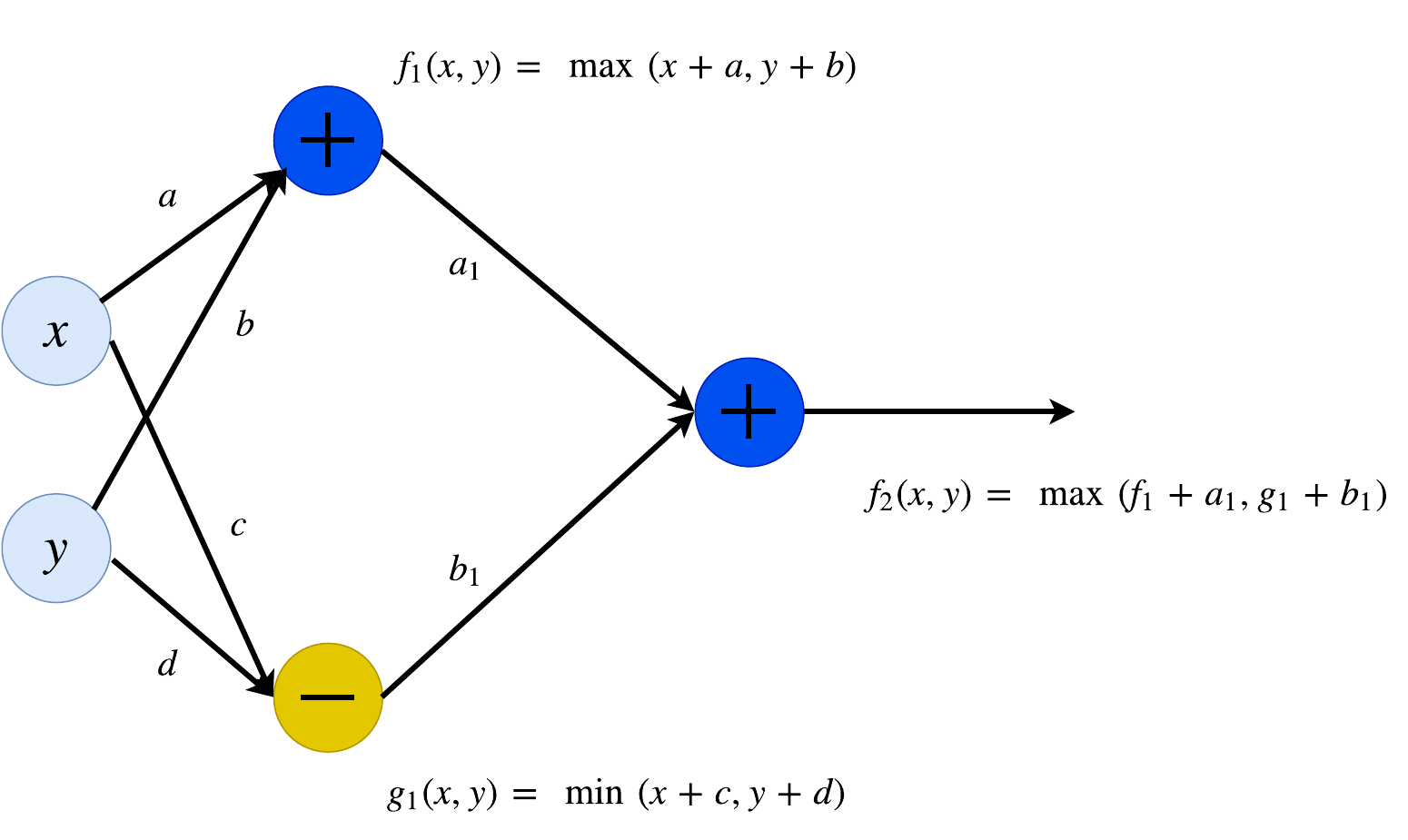}
    \caption{A network of architecture $D_1E_1 \rightarrow D_1$}
    \label{fig:d1e1_d1}
\end{figure}
For simplicity, we will assume $2$-dimensional input. Suppose that the outputs from the first layer are $f_1(x, y)$ and $g_1(x, y)$ where $f_1$ is the output of a dilation neurone and $g_1$ is the output of an erosion neurone (see Figure~\ref{fig:d1e1_d1}). We write
\[
    f_1(x, y) = \max\{x + a, y + b\}, \quad g_1(x, y ) = \min\{ x + c, y + d\}.
\]
After the second layer consisting of a single dilation neurone, we get the output
\[
    f_2(x, y) = \max\{ f_1 + a_1, g_1 + b_1\}.
\]
Note that
\begin{align*}
    f_2(x, y) \le e &\iff f_1 + a_1 \le e \text{ and } g_1 + b_1 \le e \\
    &\iff f_1 \le e - a_1 \text{ and } g_1 \le e - b_1 \\
    &\iff (x + a \le e - a_1 \text{ and } y + b \le e - a_1) \text{ and }  (x + c \le e - b_1 \text{ or } y + d \le e - b_1) \\
    &\iff (x, y) \in (-\infty, \gamma_1] \times (-\infty, \gamma_2] \cap ((-\infty, \gamma_3] \times \R \cup  \R \times (-\infty, \gamma_4]) \\
    &\iff (x, y) \in (-\infty, \gamma_1 \wedge \gamma_3] \times (-\infty, \gamma_2] \cup (-\infty, \gamma_1] \times (-\infty, \gamma_2 \wedge \gamma_4].
\end{align*}
Note that $\gamma_1 \le \gamma_3 \iff a_1 + a \ge b_1 + c$ and $\gamma_2 \le \gamma_4 \iff a_1 + b \ge b_1 + d$.
Therefore, if $a_1 + a \ge b_1 + c$ and $a_1 + b \ge b_1 + d$, then
\[
    f_2^{-1}(-\infty, e] = (-\infty, \gamma_1] \times (-\infty, \gamma_2].
\]
Thus in this case $f_2$ can be realized in the architecture $D_1 E_0$.

If, however, $a_1 + a < b_1 + c$ and $a_1 + b < b_1 + d$, then
\[
    f_2^{-1}(-\infty, e] = (-\infty, \gamma_3] \times (-\infty, \gamma_2] \cup (-\infty, \gamma_1] \times (-\infty, \gamma_4],
\]
which is not realizable as the sublevel set of a function of $D_1E_0$ architecture.
\vskip10pt
\noindent
(iii) 
The proof is a simple modification of the proof of (ii). For $\alpha > 0$,
\begin{align*}
    \alpha f_2(x, y) \le e &\iff f_1 + a_1 \le \frac{e}{\alpha} \text{ and } g_1 + b_1 \le \frac{e}{\alpha} \\
    &\iff f_1 \le \frac{e}{\alpha} - a_1 \text{ and } g_1 \le \frac{e}{\alpha} - b_1 \\
    &\iff (x + a \le \frac{e}{\alpha} - a_1 \text{ and } y + b \le \frac{e}{\alpha} - a_1) \\
    & \text{ and } (x + c \le \frac{e}{\alpha} - b_1 \text{ or } y + d \le \frac{e}{\alpha} - b_1) \\
    &\iff (x, y) \in (-\infty, \gamma_1] \times (-\infty, \gamma_2] \cap ((-\infty, \gamma_3] \times \R \cup  \R \times (-\infty, \gamma_4]) \\
    &\iff (x, y) \in (-\infty, \gamma_1 \wedge \gamma_3] \times (-\infty, \gamma_2] \cup (-\infty, \gamma_1] \times (-\infty, \gamma_2 \wedge \gamma_4].
\end{align*}
Note that $\gamma_1 \le \gamma_3 \iff a_1 + a \ge b_1 + c$ and $\gamma_2 \le \gamma_4 \iff a_1 + b \ge b_1 + d$.

Therefore, if $a_1 + a \ge b_1 + c$ and $a_1 + b \ge b_1 + d$, then
\[
    (\alpha f_2)^{-1}(-\infty, e] = (-\infty, \gamma_1] \times (-\infty, \gamma_2].
\]
Thus in this case $f_2$ can be realized in the architecture $D_1 E_0 \rightarrow L$.

If, however, $a_1 + a < b_1 + c$ and $a_1 + b < b_1 + d$, then
\[
    (\alpha f_2)^{-1}(-\infty, e] = (-\infty, \gamma_3] \times (-\infty, \gamma_2] \cup (-\infty, \gamma_1] \times (-\infty, \gamma_4],
\]
which is not realizable as the sublevel set of a function of $D_1E_0 \rightarrow L$ architecture.

Sub-level sets of the $D_1E_0 \rightarrow L$ architecture. For $\beta > 0$
\begin{align}
    \beta \max\{ x + u, y + v \} \le e & \iff x \le \frac{e}{\beta} - u \text{ and } y \le \frac{e}{\beta} - v.
\end{align}

Equating $\frac{e}{\beta} - u = \frac{e}{\alpha} - a_1 - a$, $\frac{e}{\beta} - v = \frac{e}{\alpha} - a_1 - b$, we can see that one can take $\beta  = \alpha, u = a + a_1, v = b + a_1$ to realize the function $\alpha f_2$ in the $D_1E_0 \rightarrow L$ architecture.

(iv) It can be proved in the same way as (ii)
\end{proof}

\section{Proof of Proposition 1: Morphological block as a sum of hinge functions}
\label{sec:proof_lemma1}
\begin{proposition}
\label{th:gx_sum_hinge}
The function computed by a single morphological block with $n$ dilation and $m$ erosion neurons followed by a linear combination layer computes $\mathcal{M}(\vx)$, which is a sum of multi-order hinge functions.
\end{proposition}
\begin{proof}
As defined in the main paper the computed $\mathcal{M}(\vx)$ has the following form.
\begin{equation}
    \mathcal{M}(\vx)=\sum_{i=1}^{n} \omega_i^{+} z_{i}^{+} + \sum_{j=1}^{m}  \omega_j^{-} z_{j}^- ,
\end{equation}
where $z_{i}^+$ and $z_{j}^-$ are the output of $i^{th}$ dilation neuron and $j^{th}$ erosion neuron, respectively and $\omega_i^{+}$ and $\omega_j^{-}$ are the weights of the the linear combination layer. Replacing the $z_i^+$ and $z_j^-$ with their expression, the equation becomes the following.
\begin{equation}
    \mathcal{M}(\vx)=\sum_{i=1}^{n}\omega_i^{+} \max_{k}\{x'_k+s_{ik}^+\}  + \sum_{i=1}^{m} -\omega_i^{-} \max_{k}\{s_{ik}^- -x'_k\},
\end{equation}
where $s_{ik}^+$ and $s_{ik}^-$ denote the $k^{th}$ component of the $i^{th}$ structuring element of dilation and erosion neurons, respectively. The above equation can be further expressed in the following form,
\begin{equation}
    \mathcal{M}(\vx)=\sum_{i=1}^{n} \alpha_{i}^+\max_{k}\{\theta{i}^+x'_{k} +\rho_{ik}^+\} \\+ \sum_{i=1}^{m} \alpha_{i}^-\max_{k}\{\theta_{i}^-x'_{k} +\rho_{ik}^-\},
    \label{eq:g_phi}
\end{equation}
Where $\theta_i^+$, $\theta_i^-$, $\rho_{ik}^+$ and $\rho_{ik}^-$ are defined in the following way
\begin{align*}
  \theta_i^+=
  \begin{cases} 
  \omega_{i}^+ & \text{if } \omega_{i}^+\ge0  \\
  -\omega_{i}^+ & \text{if } \omega_{i}^+<0
  \end{cases}
  &&
  \theta_i^-=
  \begin{cases} 
  -\omega_{i}^- & \text{if } \omega_{i}^-\ge0  \\
  \omega_{i}^- & \text{if } \omega_{i}^-<0
  \end{cases}
  \\
  \rho_{ik}^+=
  \begin{cases} 
  s_{ik}^+\omega_{i}^+ & \text{if } \omega_{i}^+\ge0  \\
  -s_{ik}^+\omega_{i}^+ & \text{if } \omega_{i}^+<0
  \end{cases}
    &&
  \rho_{ik}^-=
  \begin{cases} 
  s_{ik}^-\omega_{i}^- & \text{if } \omega_{i}^-\ge0  \\
  -s_{ik}^-\omega_{i}^- & \text{if } \omega_{i}^-<0
  \end{cases}
  \\
  \alpha_{i}^+=
  \begin{cases} 
  1 & \text{if } \omega_{i}^+\ge0  \\
  -1 & \text{if } \omega_{i}^+<0
  \end{cases}
  &&
  \alpha_{i}^-=
  \begin{cases} 
  -1 & \text{if } \omega_{i}^-\ge0  \\
  1 & \text{if } \omega_{i}^-<0
  \end{cases}
\end{align*}

Now, without any loss of generality, we can write \eqref{eq:g_phi} as follows
\begin{align}
    \mathcal{M}(\vx)  & = \sum_{i=1}^{m+n} \alpha_{i}\max_{k}(\theta_{i}x'_{k} +\rho_{ik})
    \label{eq:finaleq}
\end{align}
where 
\begin{align*}
 \theta_{i}=
  \begin{cases} 
  \theta{i}^+ & \text{if $i\leq n$ }   \\
  \theta_{i-n}^-  & \text{if $ n < i \leq m+n$ } 
  \end{cases}
  \\
 \rho_{ik}=
  \begin{cases} 
  \rho_{ik}^+ & \text{if $i\leq n$ }   \\
  \rho_{(i-n)k}^-  & \text{if $ n < i \leq m+n$ } 
  \end{cases}
  \\
 \alpha_{i}=
  \begin{cases}
  \alpha_{i}^+ & \text{if $i\leq n$ }   \\
  \alpha_{(i-n)}^-  & \text{if $ n < i \leq m+n$ }
  \end{cases}
\end{align*}

Finally, we can rewrite equation~\ref{eq:finaleq} as 
\begin{equation}
    \mathcal{M}(\vx) = \sum_{i=1}^{l} \alpha_{i} \phi_{i}(\vx),
     \label{eq:final2}
\end{equation}
where $l=m+n$, $\alpha_{i} \in \{1,-1\}$  and $\phi_i(\vx)$'s are of the following form
\begin{align}
    & \phi_{i}(\vx)= \max_{k}(\vv_{ik}^{T}\vx'+\rho_{ik}), 
    \label{eqstd1}
\end{align} 
with
\begin{align}
 & v_{ikt} = 
  \begin{cases} 
  \beta_{i} & \text{if } t=k  \\
  0  & \text{if } t\neq k
  \end{cases}
\end{align}
In \eqref{eqstd1}, $\vv_{ik}^{T}\vx'+\rho_{ik}$ is affine and $\alpha_{i} \phi_{i}(\vx)$ is a $d$-order hinge function. Hence $\sum_{i=1}^{l} \alpha_{i} \phi_{i}(\vx)$  i.e., $\mathcal{M}(\vx)$ represents  sum of multi-oder hinge function. 
\end{proof}

\section{Number of Hyperplanes}
Since a morphological block computes a sum of hinge functions, it can potentially learn a large number of hyperplanes. 
The function $\mathcal{M}(\vx)$ learned by a single-layer Morphological network may also be expressed in the following form:
\begin{equation}
    \mathcal{M}(\vx) = \sum_{i=1}^{l}\alpha_{i}\max_{k}\{\theta_{k}x_{k}+\rho_{ik}\},
    \label{eq:num_d_plane}
\end{equation}
where $\alpha_i$, $\theta_{k}$, $\rho_{ik} \in \R$. We see that $\mathcal{M}(\vx)$ is a sum of $l$ functions, each of which computes $\max$ over the linearly transformed elements of $\vx$. Since the $\max$ is computed over the (transformed) elements of $\vx$, each $\max$ operation selects only one element of $\vx$. So, the computed $\mathcal{M}(\vx)$ may not contain all the elements of $\vx$ and the index ($k$) of the selected element varies depending on the input and the structuring element. However, if $l > d$, $\mathcal{M}(\vx)$ may contain all the elements of $\vx$. So \eqref{eq:num_d_plane} can be rewritten as 
\begin{equation}
    \mathcal{M}(\vx) = \alpha_1(\theta_1 x_{k_1} + \rho_{1k_1}) + \alpha_2(\theta_2 x_{k_2} + \rho_{2k_2}) + \cdots 
    + \alpha_l(\theta_l x_{k_l} + \rho_{lk_l}).
\end{equation}
where $x_{k_i}$ represents any one of the $d+1$ elements of $\vx$ selected by $i$-th neuron by $\max$ operation depending on structuring element $\vs_i$. So each $x_{k_i}$ is chosen from $d+1$ elements. Therefore, depending on which element of $\vx$ gets selected by each neuron, $\mathcal{M}(\vx)$ forms one of the $(d+1)^l-1$ hyperplanes. The $-1$ occurs in the number of hyperplanes because on one occasion only limiter or bias is selected. Note that some of these hyperplanes must be parallel to some axes. 
For $\mathcal{M}(\vx)$ to form a hyperplane that is not parallel to any of the axes, all elements of $\vx$ must get selected by some $\max$ functions or other. This occurs in $d! \times \binom{l}{d}$ ways. The remaining $l-d$ number of elements $x_{k_i}$'s are repeat selection by some functions. So, there can be almost $d! \times \binom{l}{d} \times {(d+1)^{l-d}}$ hinging hyperplanes that are not parallel to any of the axes. 

\section{One morphological block and function approximation}
A single morphological block represents a sum of hinge functions. However, it is not clear if all hinge functions can be represented by a single morphological block. 
In a numerical study, we have tried to approximate the hinge function $\max(x+y, 0)$ using a single morphological block by varying the number of dilation/erosion neurons. We have generated values of the function in the square $[-5, 5] \times [-5, 5]$, and trained the network with mean squared error (MSE) loss. In  Figure~\ref{fig:morph-block}, we have plotted the MSE loss (after convergence) against the number of morphological neurons used. It is seen that a single morphological block is unable to reduce the error unless we use additional bias or limiter terms in the morphological neurons. However, we do not know theoretically if having additional bias terms in morphological operations help in universal approximation.

\begin{figure}[htb]
    \centering
    \includegraphics[width=0.71\linewidth]{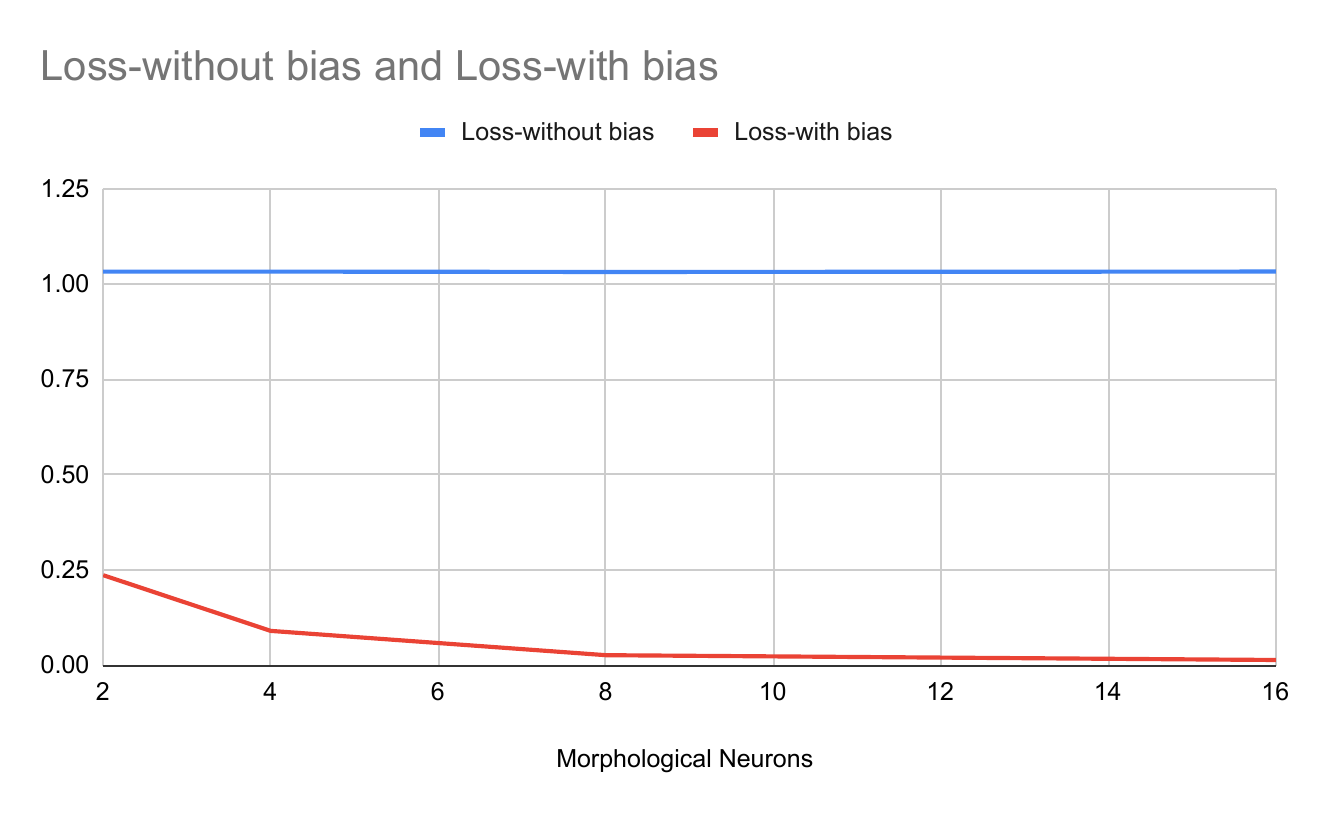}
    \caption{Graph of approximation loss  with varying  morphological neurons in a single morphological block.}
    \label{fig:morph-block}
\end{figure}

\section{Universal Approximation by two Morphological blocks}
Here we have shown that two sequential Morphological blocks can approximate any continuous functions. First, we have shown that any hyperplane can be represented by a single morphological block. After that, we have shown the universal approximation using two morphological blocks.

\begin{lemma}\label{lem:hyp_rep_main}
Let $K$ be a compact subset of $\R^d$. Then, over $K$, any hyperplane $w^\top\vx + b$ can be represented as an affine combination of $d$ dilation neurons which only depend on $K$.
\end{lemma}
\begin{proof}
Since we are in a compact set, there exists $C > 0$ such that $|x_\ell| \le C$ for any $1 \le \ell \le d$. Where $x_\ell$ is each element of $\vx$. Take 
\[
    s_{\ell} =  -3C \mathbf{1}_{d} + 3C e_{\ell, d}, 1 \le \ell \le d,
\]
where $\mathbf{1}_{d}$ is the vector of all ones and $e_{\ell, d}$ is the $\ell$-th unit vector in $\R^{d}$. Then all but the $\ell$-th coordinate of $s_{\ell}$ are $-3C$, while the $\ell$-th coordinate is $0$. Then note that, for any $ \vx \in K$, and $1 \le \ell \le d$,
\begin{align*}
    x_{\ell} + s_{\ell, \ell} = x_{\ell} \ge - C > -2C &= C - 3C \\
    &\ge x_{j} - 3C = x_j + s_{\ell, j},
\end{align*}
for any $j \ne \ell$. It follows that for any $x \in K$, and $1 \le \ell \le d$,
\[
    \vx \oplus s_{\ell} = x_\ell.
\]

Now given any hyperplane $w^\top\vx +b$, we can express it exactly as a linear combination of dilation neurons over $K$:
\[
    w^\top\vx + b = \sum_{\ell = 1}^d w_{\ell} x_{\ell} + b = \sum_{\ell = 1}^d w_{\ell} (\vx \oplus s_{\ell}) + b.
\]
This completes the proof.
\end{proof}

\begin{lemma}[lemma 1 of main paper]
Any linear combination of hinge functions $\sum_{i = 1}^m \alpha_i h^{(k_i)}(\vx)$ can be represented over any compact set $K$ as a two sequential  morphological  block consisting of dilation neurons only.
\end{lemma}
\begin{proof}
Let $B = \max_{1 \le i \le m}\sup_{\vx \in K} |h^{(k_i)}(\vx)|$. We now give the architecture of the desired Morph-Net.
\begin{enumerate}
    \item The first dilation-erosion layer has exactly $d$ dilation neurons given by $\vx \oplus s_{\ell}, 1 \le \ell \le d$.
    \item The first linear combination layer has $k = \sum_{i = 1}^m (k_i + 1)$ neurons, with the $i$-th block of $(k_i + 1)$ neurons outputting the constituent hyperplanes of $h^{(k_i)}(\vx)$. This can be done by Lemma~\ref{lem:hyp_rep_main}.
    \item The second dilation-erosion layer just has $m$ dilation neurons, each outputting a hinge function. The $\ell$-th neuron is constructed as follows: Write any $\vy \in \R^k$ as $(\vy_1^\top, \ldots, \vy_m^\top)^\top$ where $\vy_j = (y_{j,1}, \ldots, y_{j, k_j + 1})^{\top}$.  We want the output of the $\ell$-th neuron to be \\ $\max_{1 \le v \le k_{\ell} + 1} y_{\ell, v}$. So we take $\vt_\ell = (\vt_{\ell, 1}^\top, \ldots, \vt_{\ell, m}^\top)^\top$, where $\vt_{\ell, j} = -3B \mathbf{1}_{k_j + 1}$ for $j \ne \ell$, and $\vt_{\ell, \ell} = \mathbf{0}_{k_\ell + 1}$. Then, for any $j \ne \ell$, $1 \le u \le k_{j} + 1$, and $1 \le v \le k_{\ell} + 1$, we have
    \begin{align*}
        y_{j, u} + t_{\ell, j, u} = y_{j, u} - 3B &\le B - 3B \\
        &= -2B \\
        &< -B \\
        &\le y_{\ell, v} = y_{\ell, v} + t_{\ell, \ell, v}. 
    \end{align*}
    It follows that $\vy \oplus \vt_\ell = \max_{1 \le v \le k_{\ell} + 1} y_{\ell, v}$. With this construction, the outputs of the second dilation-erosion layer are the $m$ numbers $h^{(k_i)}(\vx)$.
    \item The second linear combination layer just has a single neuron that combines the outputs of the previous layer in the desired way:
    \[
        \vz \mapsto \sum_{i = 1}^m \alpha_i z_i.
    \]
\end{enumerate}
This completes the proof.
\end{proof}

\section{Results}
\subsection{HIGGS Dataset}
Here also show some results on the Higgs dataset. Here the performance of the morphological block is not so good, but this could provide some hint towards improving the performance. 
\begin{figure}[!h]
    \centering
    \subfigure[input]{\includegraphics[width=0.48\linewidth]{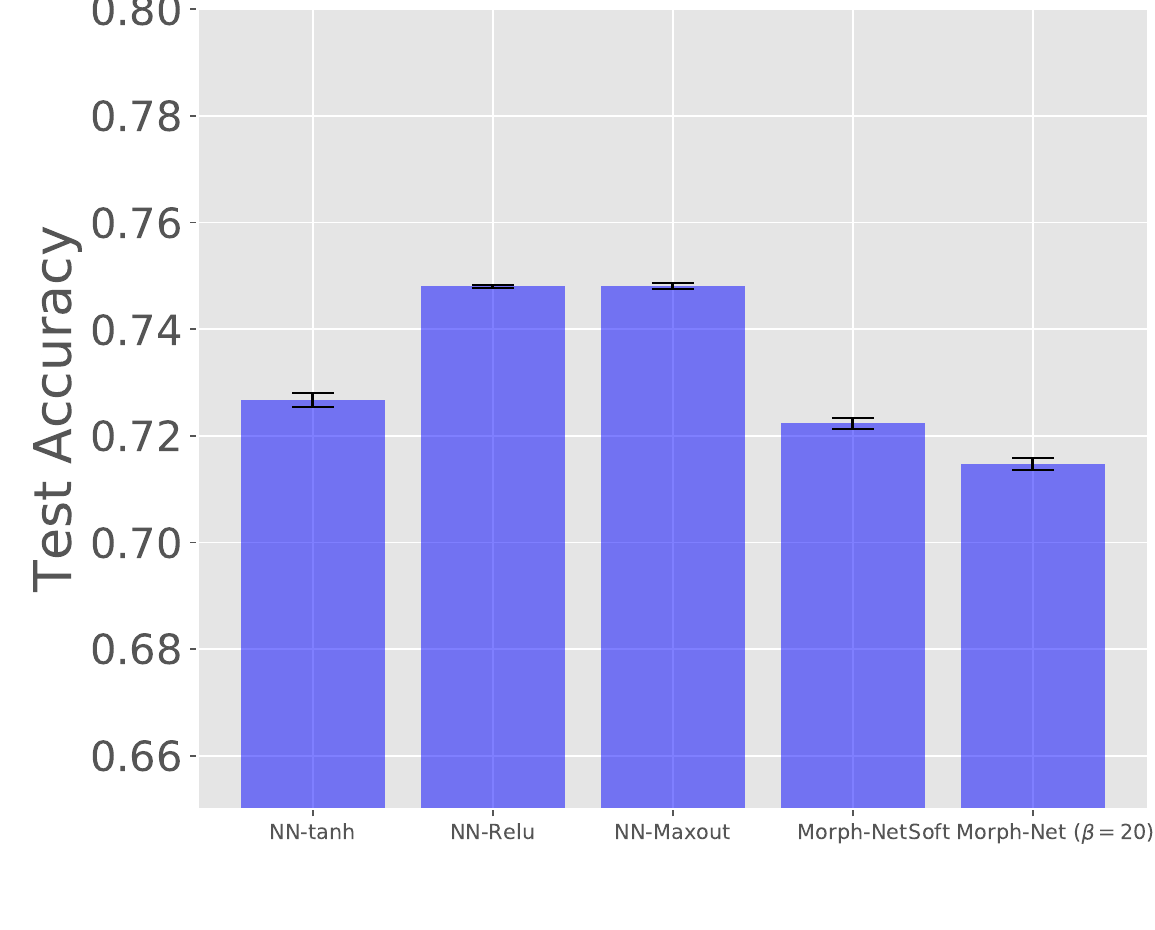}}
    \subfigure[Morph-Net]{\includegraphics[width=0.48\linewidth]{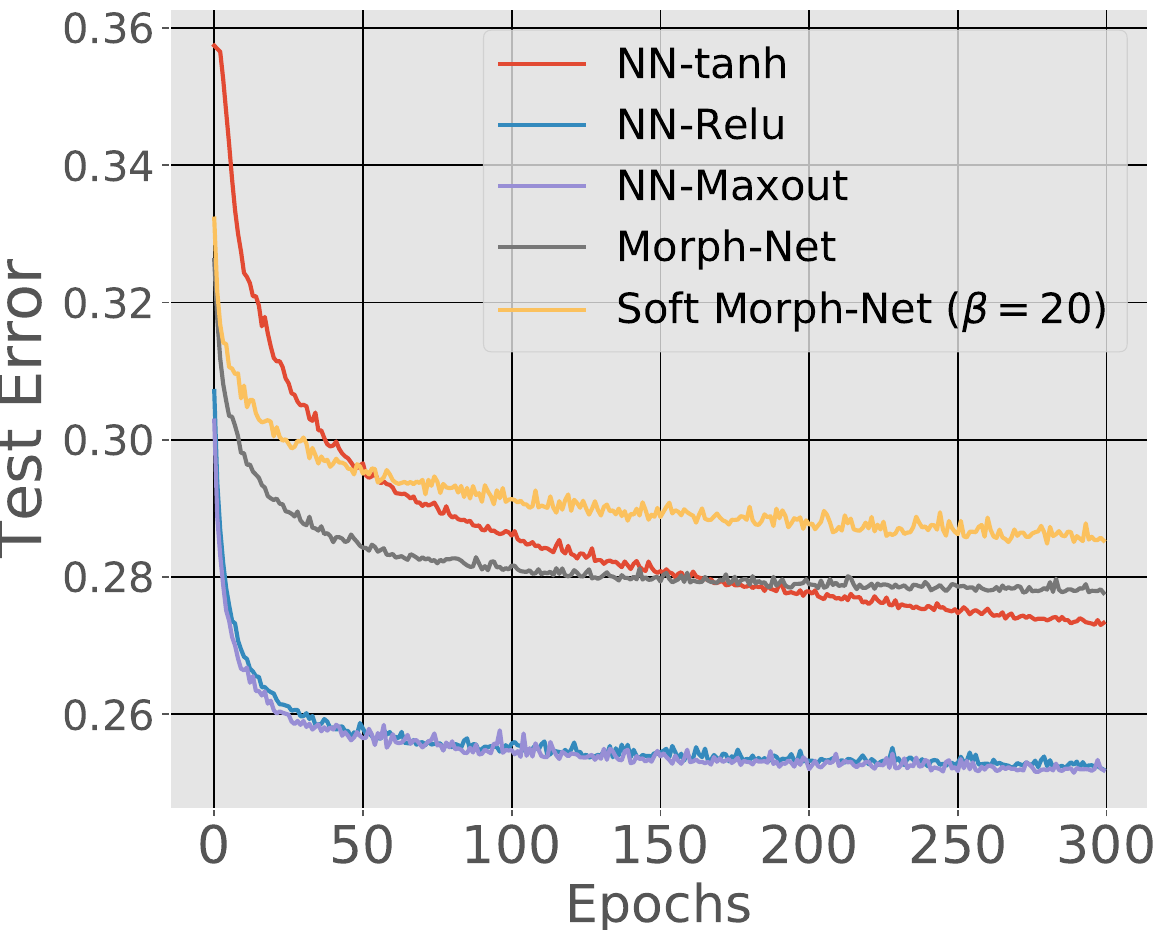}}
\caption{Results on Higgs dataset}
\end{figure}

\subsection{CIFAR10 and SVHN}
A few results of CIFAR10 and SVHN  dataset by varying number of neurons is given in figure~\ref{cifar10ex} and figure~\ref{svhnex}. 

\begin{figure}[htb]
    \centering
    \begin{tabular}{cc}
         \includegraphics[width=0.47\linewidth]{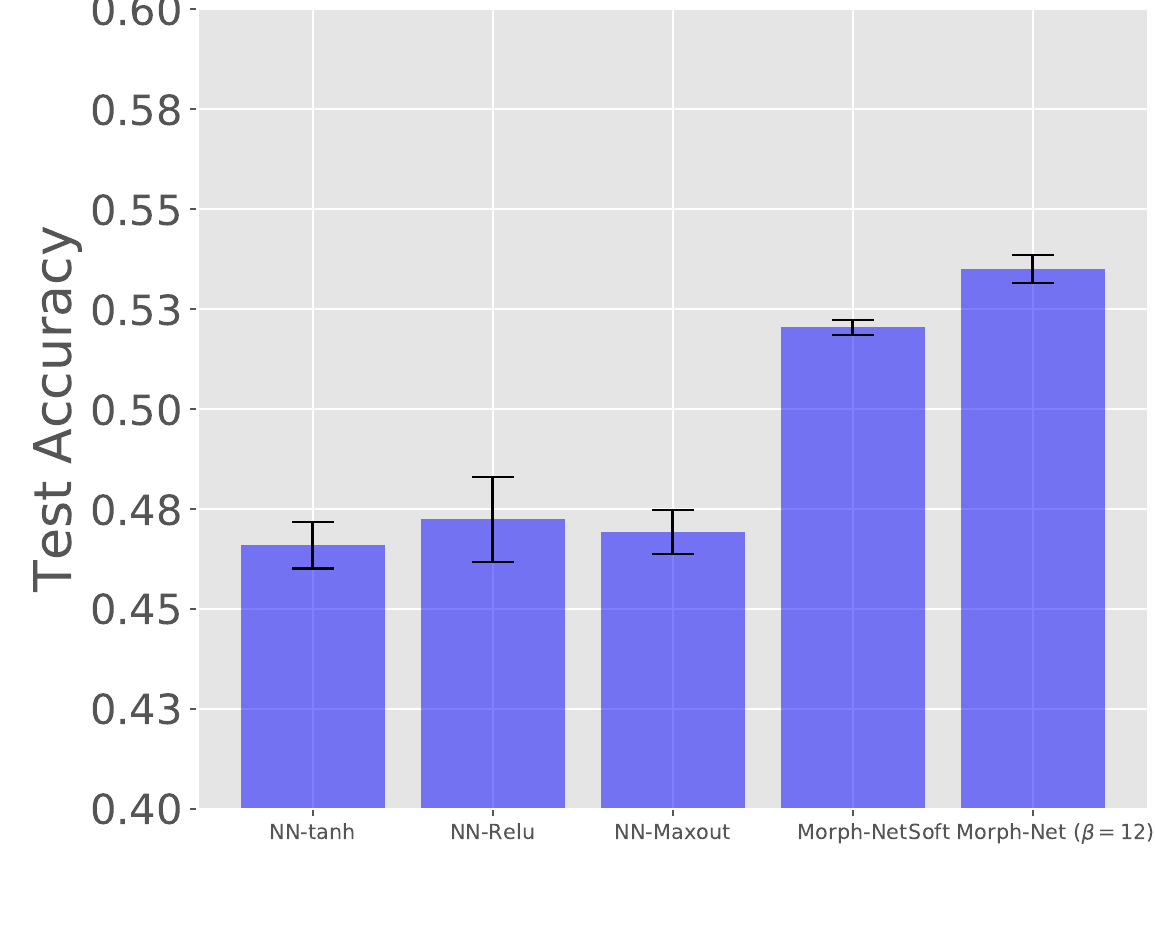} & \includegraphics[width=0.46\linewidth]{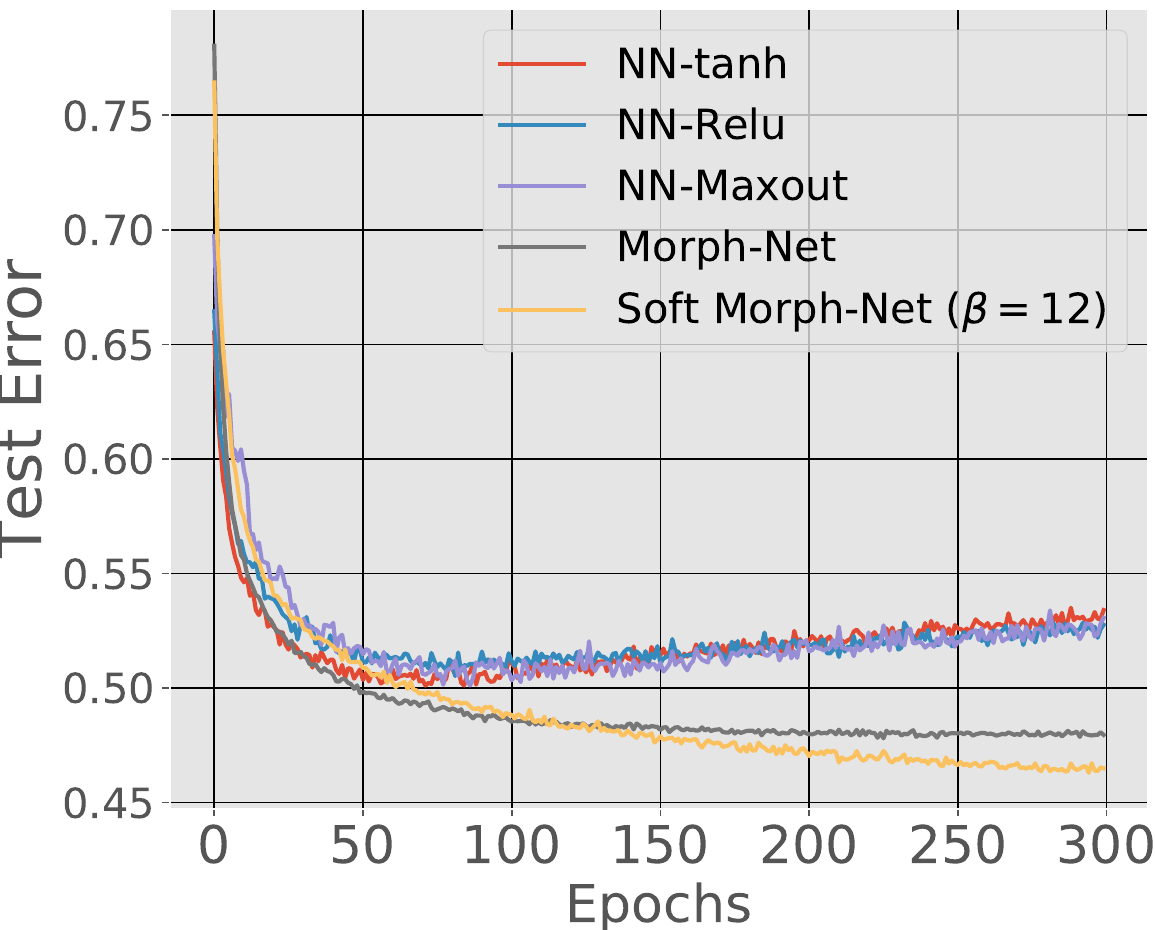} \\
         \includegraphics[width=0.47\linewidth]{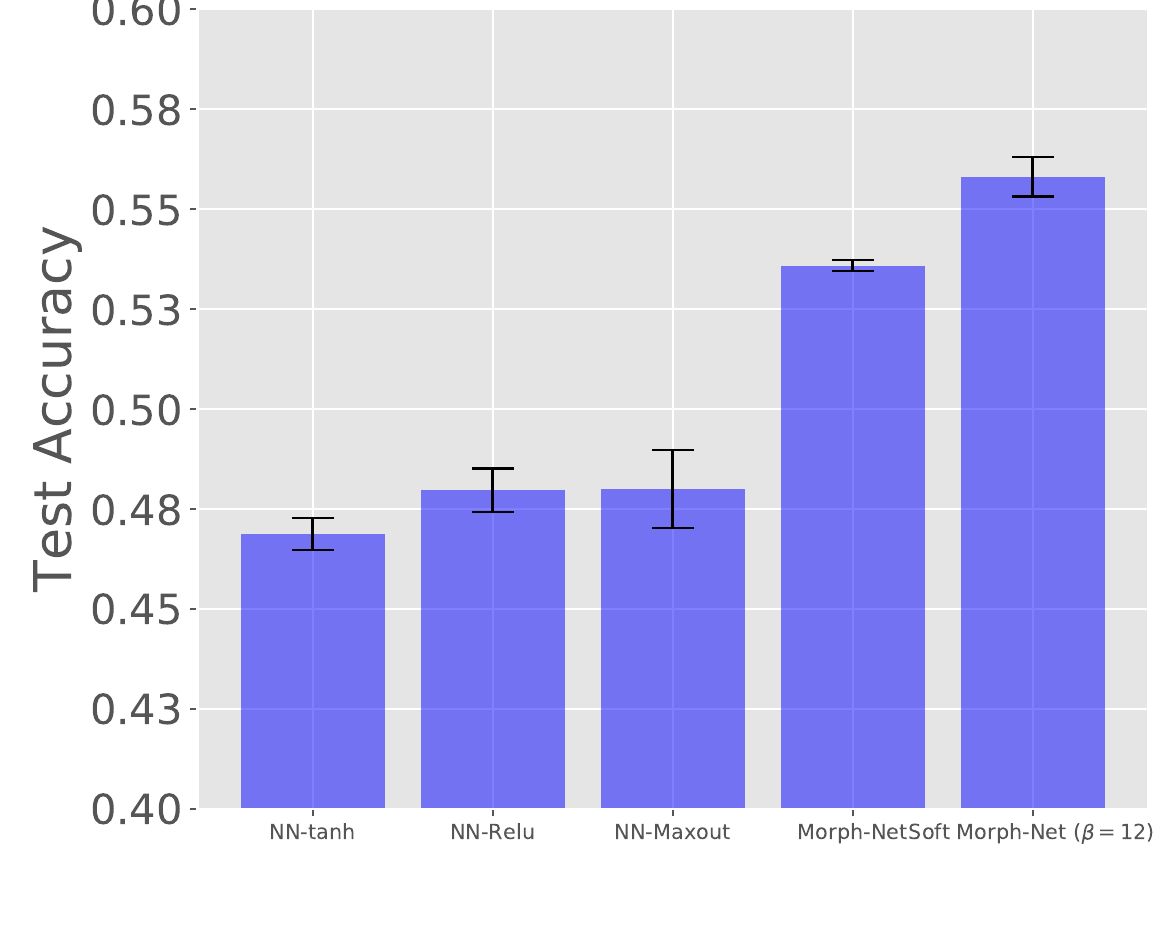}& \includegraphics[width=0.46\linewidth]{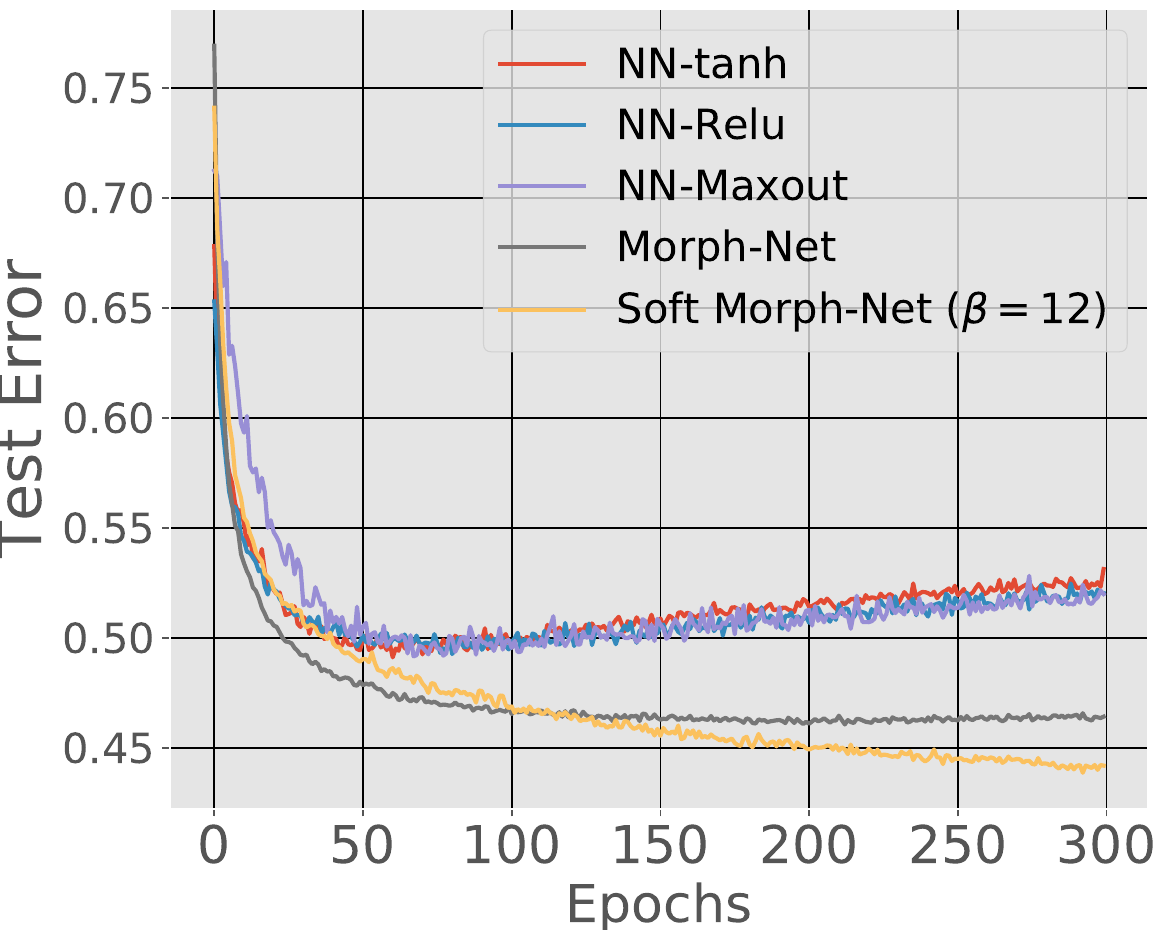} \\
         \includegraphics[width=0.47\linewidth]{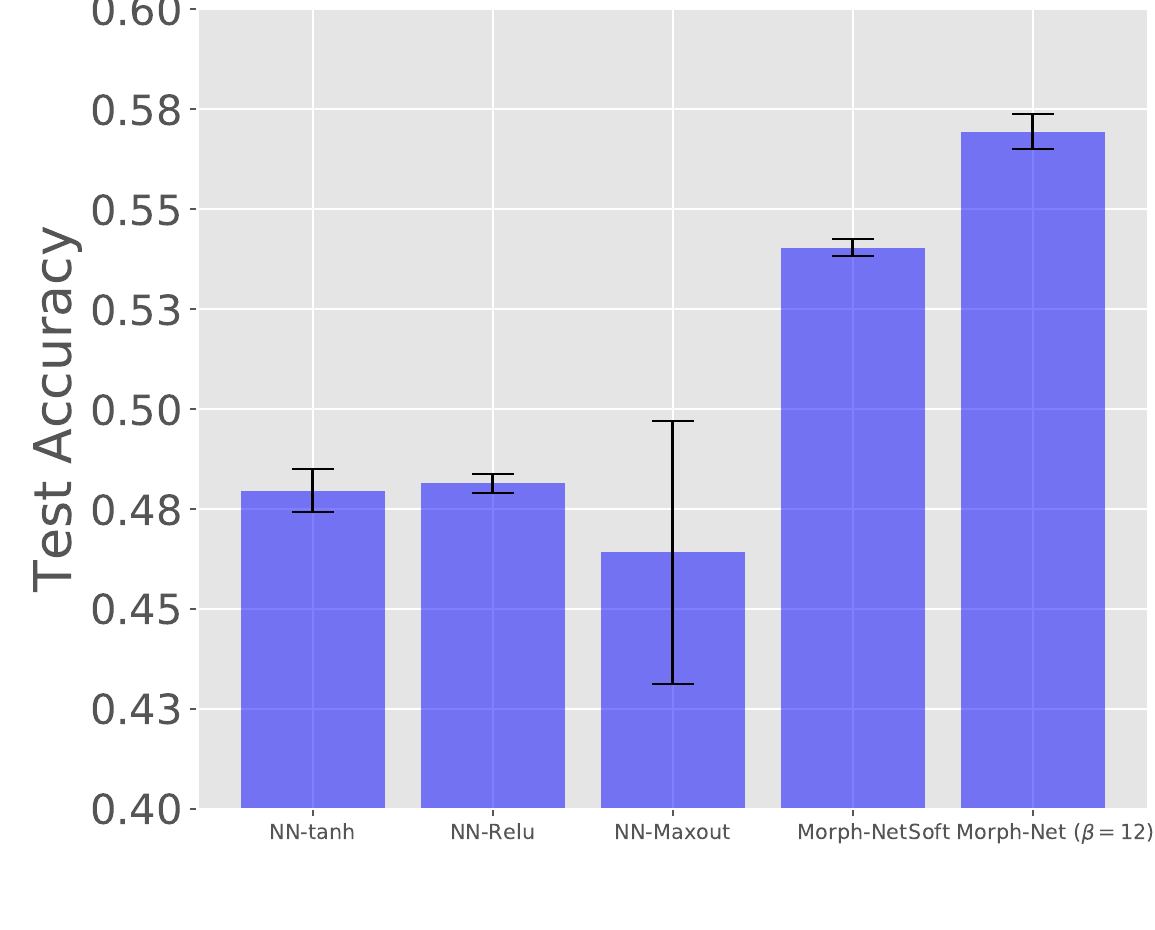} & \includegraphics[width=0.46\linewidth]{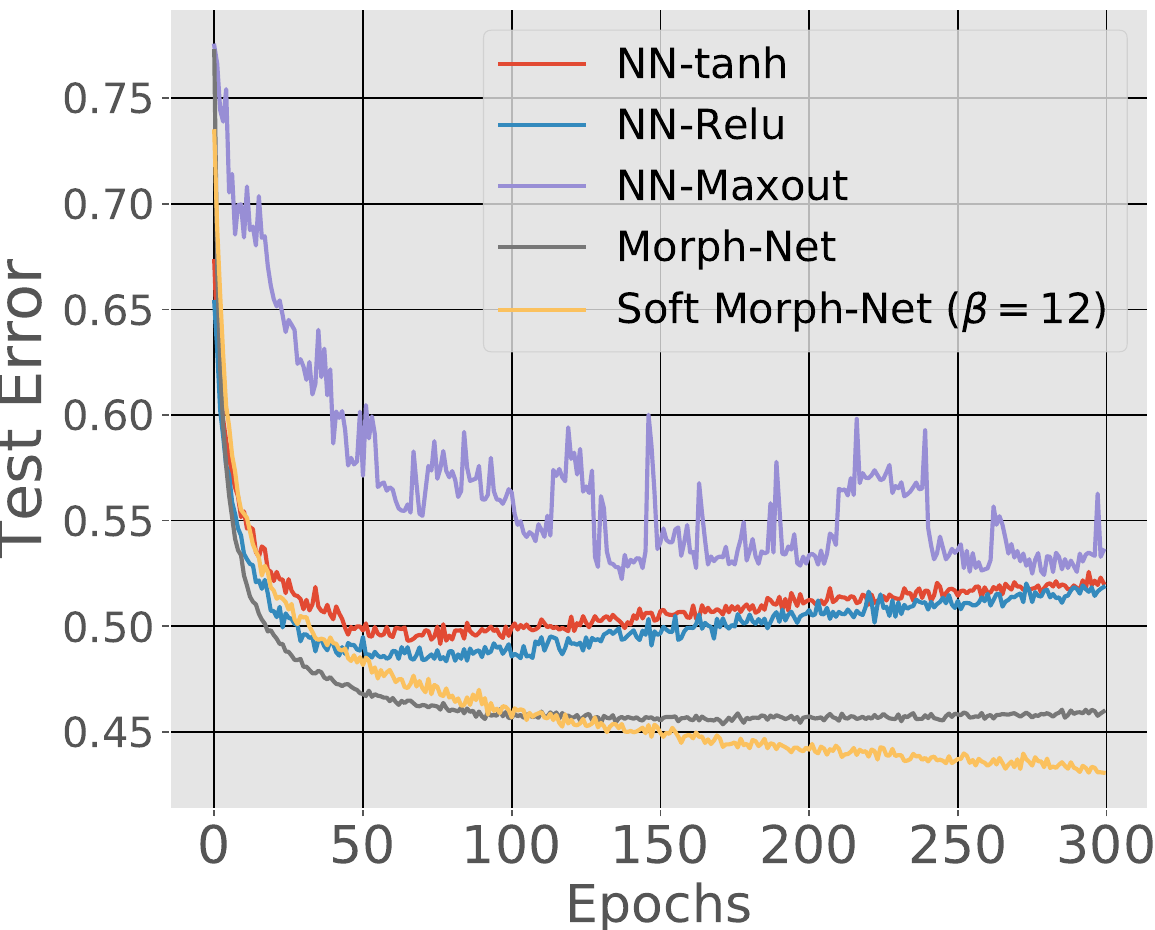} \\
         (a) Bar graph & (b) Test error vs epochs
    \end{tabular}
\caption{Results on CIFAR10 dataset, varying number of neurons in the hidden layer: l=200 (1st row), l=400 (2nd row), l=600 (3rd row) }
\label{cifar10ex}
\end{figure}

\begin{figure}[!h]
    \centering
    \begin{tabular}{cc}
         \includegraphics[width=0.47\linewidth]{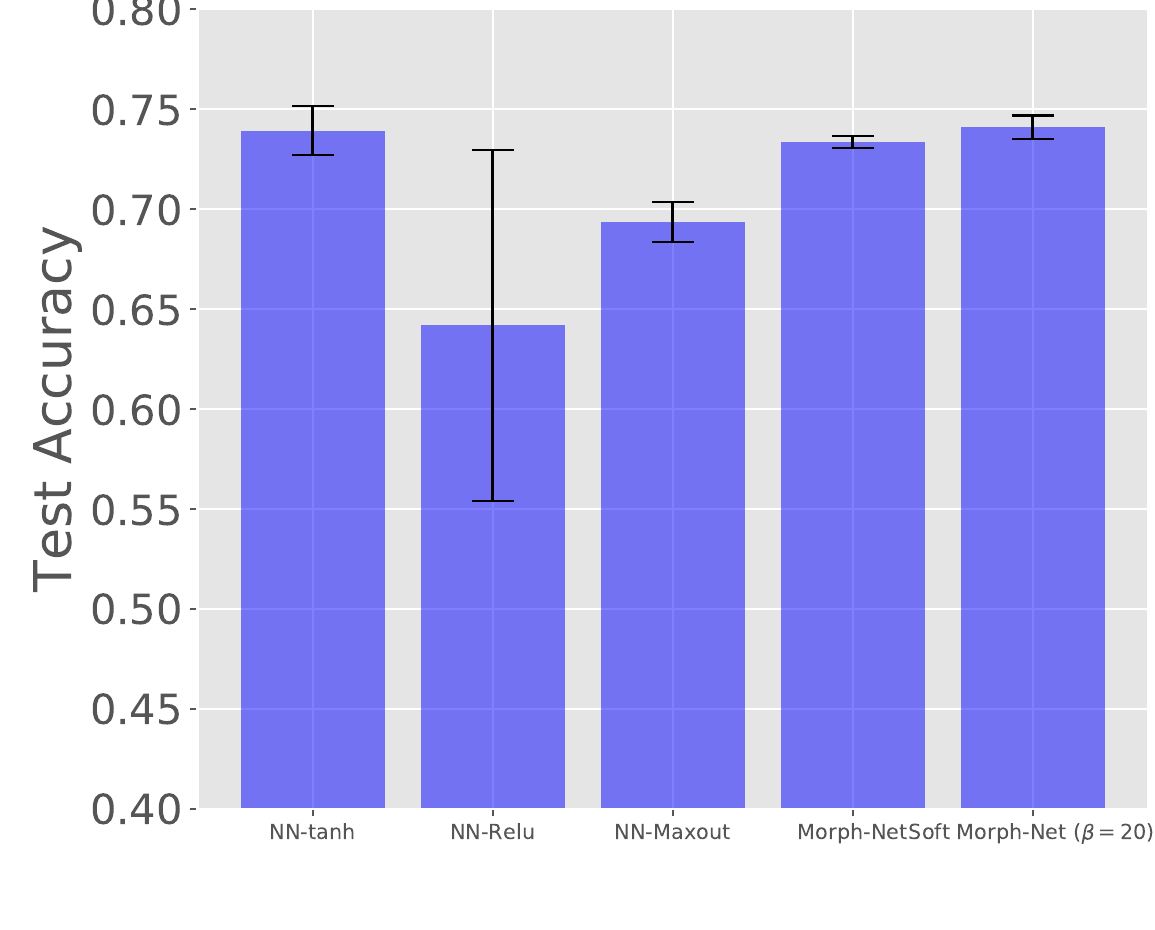} & \includegraphics[width=0.46\linewidth]{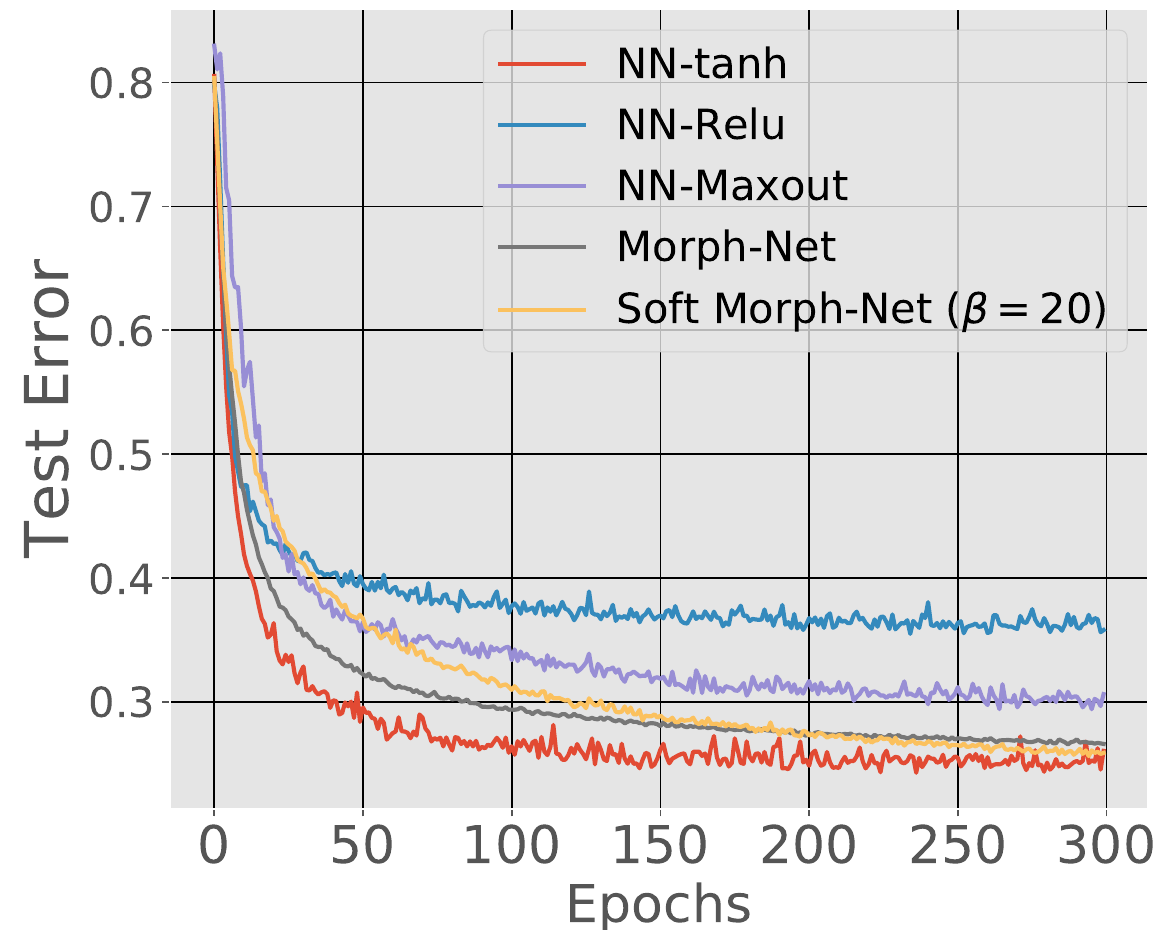} \\
         \includegraphics[width=0.47\linewidth]{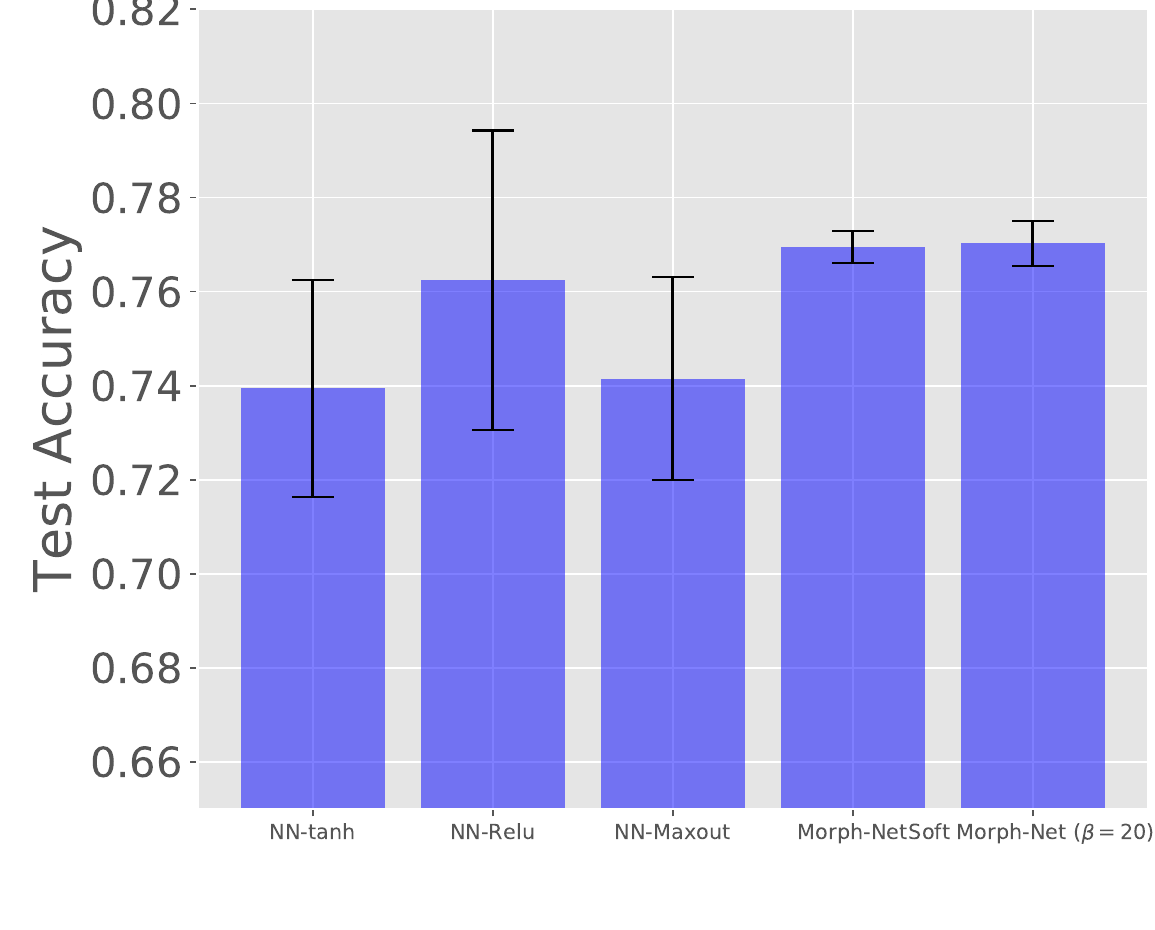} & \includegraphics[width=0.45\linewidth]{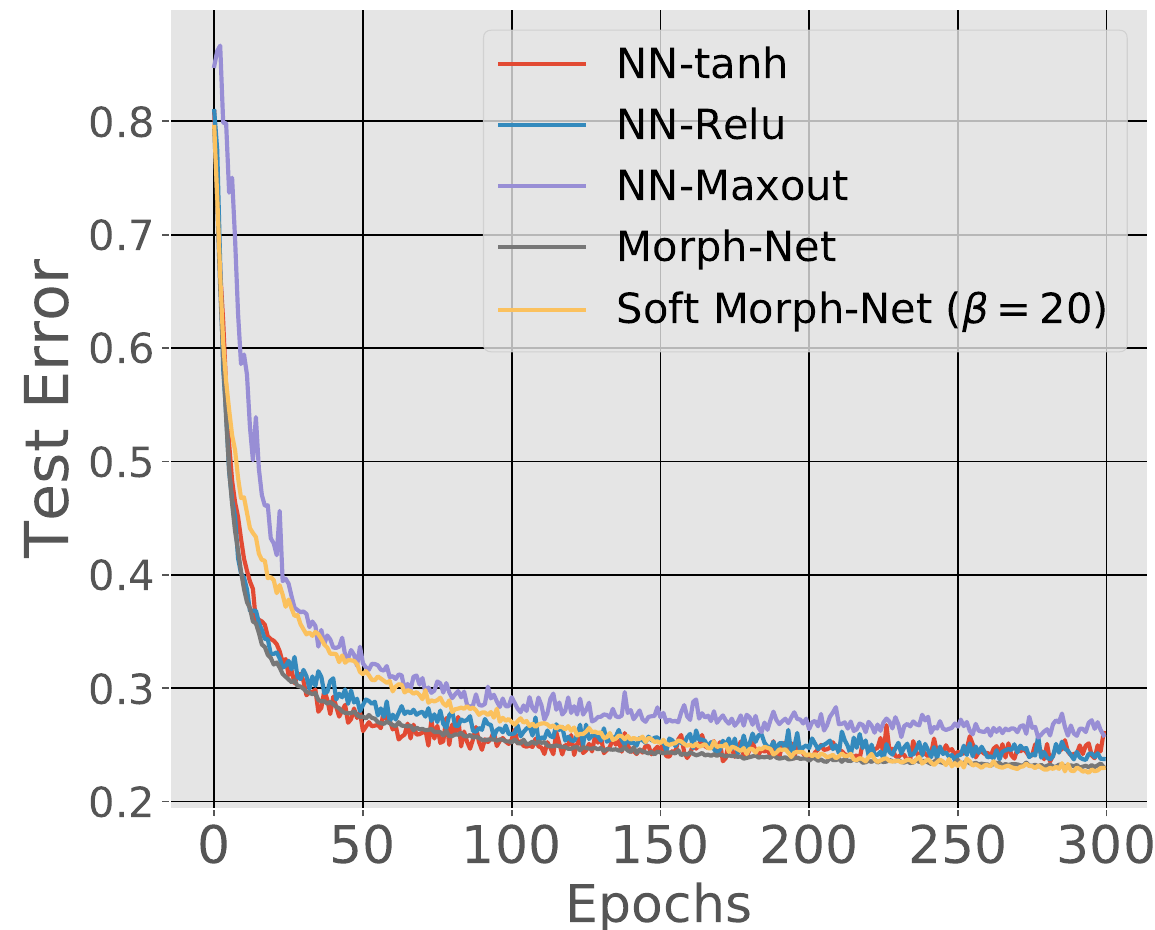} \\
         \includegraphics[width=0.47\linewidth]{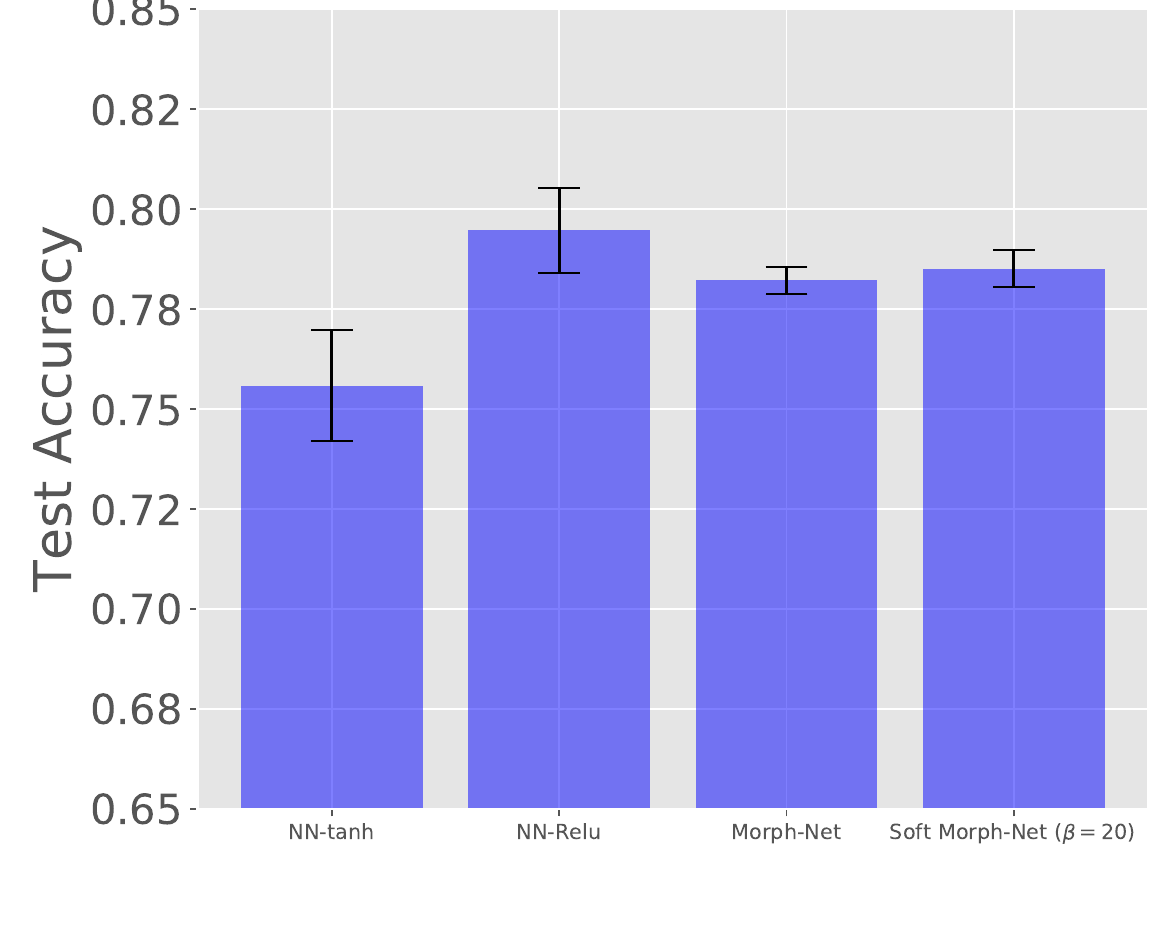} & \includegraphics[width=0.45\linewidth]{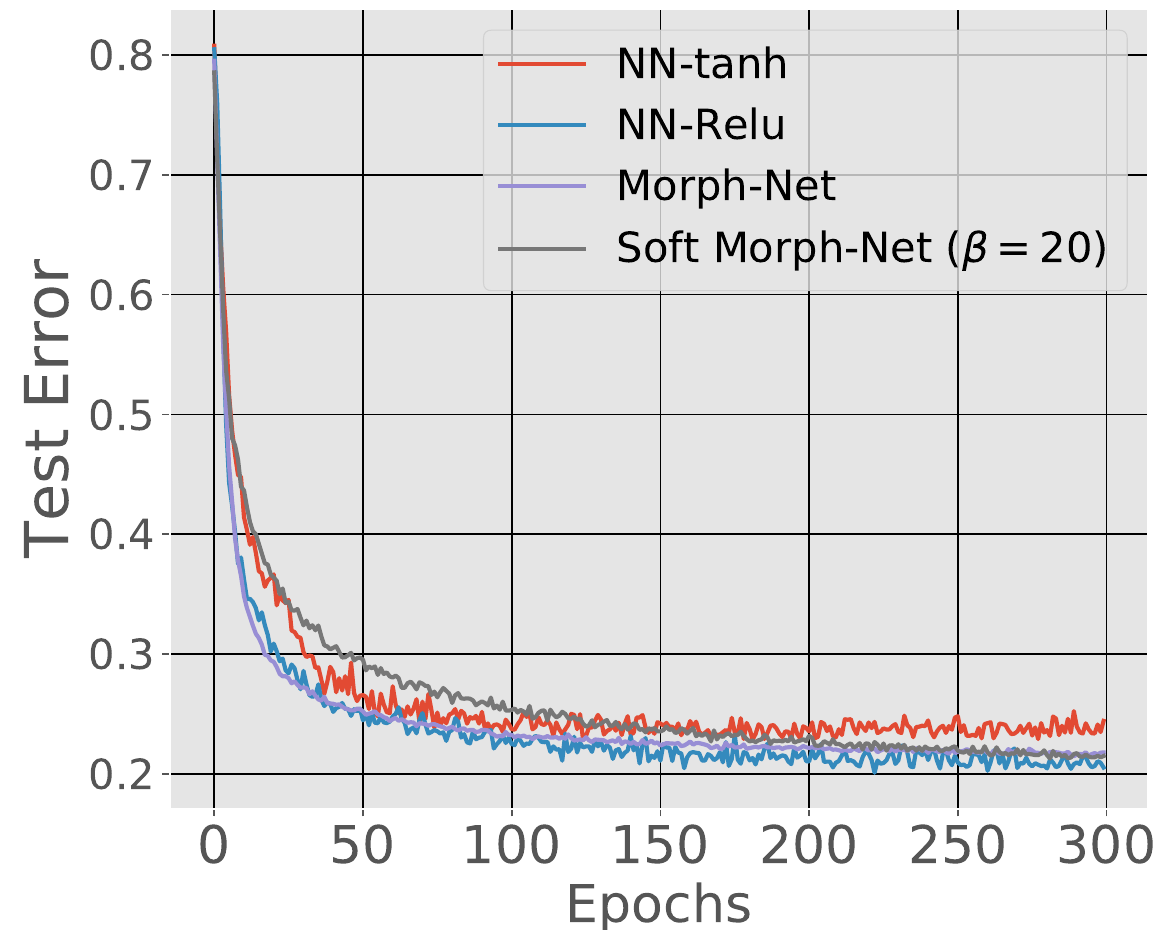} \\
         (a) Bar graph & (b) Test error vs epochs
    \end{tabular}
\caption{Results on SVHN dataset, varying number l=200 (1st row), l=400 (2nd row), l=600 (3rd row)}
\label{svhnex}
\end{figure}

\clearpage
